\documentclass[11pt]{article}

\usepackage{arxiv}
\usepackage[margin=1in]{geometry}
\usepackage{microtype}
\usepackage{graphicx}
\usepackage{subfigure}
\usepackage{booktabs} 
\usepackage{natbib}
\usepackage[affil-it]{authblk}

\usepackage{subfiles}
\usepackage[colorlinks,citecolor=blue,urlcolor=blue,linkcolor=blue,linktocpage=true]{hyperref}
\usepackage{libertine}
\usepackage{libertinust1math}
\usepackage[T1]{fontenc}

\usepackage{amsmath,amsfonts,bm}
\usepackage{amssymb,amsthm}
\usepackage{mathtools}
\usepackage{algorithm,algorithmic}
\usepackage{xcolor}
\usepackage{hyperref}
\usepackage{url}
\usepackage{etoc}
\usepackage{cancel}
\usepackage[capitalize]{cleveref}

\renewcommand{\algorithmicrequire}{\textbf{Input:}}
\renewcommand{\algorithmicensure}{\textbf{Output:}}

\crefname{proposition}{Proposition}{Propositions}
\crefname{theorem}{Theorem}{Theorems}
\crefname{lemma}{Lemma}{Lemmas}
\crefname{update_rule}{Update}{Updates}
\crefname{algorithm}{Algorithm}{Algorithms}

\def\eqref#1{equation~\ref{#1}}

\def\1{\bm{1}}

\DeclareMathAlphabet{\mathsfit}{\encodingdefault}{\sfdefault}{m}{sl}
\SetMathAlphabet{\mathsfit}{bold}{\encodingdefault}{\sfdefault}{bx}{n}

\def\cA{{\mathcal{A}}}

\def\cE{{\mathcal{E}}}

\def\cI{{\mathcal{I}}}

\def\cM{{\mathcal{M}}}
\def\cN{{\mathcal{N}}}

\def\cP{{\mathcal{P}}}

\def\cR{{\mathcal{R}}}
\def\cS{{\mathcal{S}}}

\def\cV{{\mathcal{V}}}

\def\sI{{\mathbb{I}}}

\def\sP{{\mathbb{P}}}

\DeclareMathOperator*{\argmax}{arg\,max}
\DeclareMathOperator*{\argmin}{arg\,min}

\newtheorem{theorem}{Theorem}
\newtheorem{lemma}{Lemma}

\newtheorem{definition}{Definition}
\newtheorem{proposition}{Proposition}
\newtheorem{remark}{Remark}

\newtheorem{corollary}{Corollary}

\newcommand{\RR}{\mathbb{R}}
\newcommand{\EE}{\mathbb{E}}

\newcommand{\addeq}{\addtocounter{equation}{1}\tag{\theequation}}
\makeatletter
\newcommand{\pushright}[1]{\ifmeasuring@#1\else\omit\hfill$\displaystyle#1$\fi\ignorespaces}
\newcommand{\pushleft}[1]{\ifmeasuring@#1\else\omit$\displaystyle#1$\hfill\fi\ignorespaces}
\makeatother

\newlength\tocrulewidth
\newcommand{\bn}{\mathbf{n}}

\newcommand{\lcbalg}{\textnormal{LCB}}
\newcommand{\ucbalg}{\textnormal{UCB}}
\newcommand{\greedyalg}{\textnormal{greedy}}

\newcommand{\iset}[1]{\left[#1\right]}
\newcommand{\Prb}[1]{\sP\left( #1 \right)}
\newcommand{\Prbb}[2]{\sP_{#1}\left( #2 \right)}
\newcommand{\idx}{I}
\newcommand{\I}[1]{\mathbb{I}\left\{#1\right\}}
\newcommand{\EEE}[2]{\mathbb{E}_{#1}\left[#2\right]}

\newcommand{\eg}{\emph{e.g.}}
\newcommand{\ie}{\emph{i.e.}}

\usepackage[disable]{todonotes}

\newcommand{\todoc}[2][]{\todo[size=\scriptsize,color=blue!20!white,#1]{Csaba: #2}}
\newcommand{\todot}[2][]{\todo[size=\scriptsize,color=purple!20!white,#1]{Tor: #2}}
\newcommand{\todoch}[2][]{\todo[size=\scriptsize,color=green!20!white,#1]{Chenjun: #2}}
\newcommand{\todoy}[2][]{\todo[size=\scriptsize,color=orange!20!white,#1]{Yifan: #2}}
\newcommand{\todob}[2][]{\todo[size=\scriptsize,color=yellow!20!white,#1]{Bo: #2}}

\title{On the Optimality of Batch Policy Optimization Algorithms}

\author{
{
Chenjun Xiao}$^{1,3,} \thanks{Equal contribution. Corresponding to {chenjun@ualberta.ca} and {yw4@andrew.cmu.edu}. }$ \ \ \ 
{Yifan Wu}$^{2,*}$\ \  
{Tor Lattimore}$^4$\ \ 
{Bo Dai}$^3$\ \  
{Jincheng Mei}$^{1,3}$\\
{Lihong Li}$^{5} \thanks{Work done when Lihong Li was with Google Research}$\ \ \ \ 
{Csaba Szepesvari}$^{1,4}$\ \ 
{Dale Schuurmans}$^{1,3}$
}
\affil{
$^1$University of Alberta\quad
$^2$Carnegie Mellon University\\
$^3$Google Research, Brain Team\quad
$^4$DeepMind\quad
$^5$Amazon
}

\begin{document}

\maketitle

\begin{abstract}

Batch policy optimization considers leveraging existing 
data for policy construction before interacting with an environment.
Although interest in this problem has grown significantly in recent years,
its theoretical foundations remain under-developed.
To advance the understanding of this problem,
we provide three results that characterize the limits and possibilities of batch policy optimization in the finite-armed stochastic bandit setting.
First, we introduce a class of \emph{confidence-adjusted index} algorithms that unifies optimistic and pessimistic principles in a common framework,
which enables a general analysis.
For this family, we show that \emph{any} confidence-adjusted index algorithm is minimax optimal, whether it be optimistic, pessimistic or neutral.
Our analysis reveals that instance-dependent optimality,
commonly used to establish optimality of \emph{on-line} stochastic 
bandit algorithms,
\emph{cannot be achieved by any algorithm} in the batch setting. 
In particular, for any algorithm that performs optimally in some environment, 
there exists another environment where the same algorithm
suffers arbitrarily larger regret.
Therefore, to establish a framework for distinguishing algorithms,
we introduce a new \emph{weighted-minimax} criterion
that considers the inherent difficulty of optimal value prediction.
We demonstrate how this criterion can be used to justify commonly used 
pessimistic principles for batch policy optimization. 

\end{abstract}

\begingroup

\section{Introduction}

We consider the problem of \emph{batch policy optimization}, 
where a learner must infer a behavior policy given only
access to a fixed dataset of previously collected experience,
with no further environment interaction available.
Interest in this problem has grown recently,
as effective solutions hold the promise of extracting powerful decision making strategies from years of logged experience, with important applications to many practical problems \citep{strehl11learning,swaminathan2015batch,covington2016deep,jaques2019way,levine2020offline}.

Despite the prevalence and importance of batch policy optimization, the theoretical understanding 
of this problem 
has, until recently, been rather limited. 
A fundamental challenge in batch policy optimization is the insufficient coverage of the dataset. 
In online reinforcement learning (RL), the learner is allowed to continually explore the environment to collect useful information for the learning tasks. 
By contrast, in the batch setting, the learner has to evaluate and optimize over various candidate policies based only on experience that has been collected a priori. 
The distribution mismatch between the logged experience and agent-environment interaction with a learned policy 
can cause erroneous value overestimation, which leads to the failure of standard policy optimization methods \citep{fujimoto2019off}. 
To overcome this problem, recent studies 
propose to use the \emph{pessimistic principle}, by either learning a pessimistic value function \citep{swaminathan2015batch,wu2019behavior,jaques2019way,kumar2019stabilizing,kumar2020conservative} or pessimistic surrogate \citep{BuGeBe20}, or planning with a pessimistic model \citep{KiRaNeJo20,yu2020mopo}. 
However, it still remains unclear how to maximally exploit the logged experience without further exploration.

In this paper, we investigate batch policy optimization with finite-armed stochastic bandits, and make three contributions 
toward better understanding the statistical limits of this problem. 
\emph{First}, 
we prove a minimax lower bound of $\Omega({1}/{\sqrt{\text{min}_i n_i}})$ on the simple regret for batch policy optimization with stochastic bandits, where $n_i$ is the number of times arm $i$ was chosen in the dataset. 
We then introduce the notion of a confidence-adjusted index algorithm that unifies both the optimistic and pessimistic principles in a single algorithmic framework. 
Our analysis suggests that any index algorithm with an appropriate adjustment,
whether pessimistic or optimistic, is
minimax optimal. 

\emph{Second}, we analyze the instance-dependent regret of batch policy optimization algorithms. Perhaps surprisingly, our main result shows that instance-dependent optimality, which is commonly used in the literature of minimizing cumulative regret of stochastic bandits, does not exist in the batch setting. 
Together with our first contribution, 
this finding challenges recent theoretical findings in batch RL that claim pessimistic algorithms are an optimal choice \citep[e.g.,][]{BuGeBe20,jin2020pessimism}. 
In fact, our analysis suggests that for any algorithm that performs optimally 
in some
environment, 
there must always exist
another environment where the algorithm suffers arbitrarily larger regret than an optimal strategy there. 
Therefore, any reasonable algorithm is equally optimal, or not optimal, depending on the exact problem instance the algorithm is facing. 
In this sense, for batch policy optimization,
there remains a lack of a well-defined optimality criterion that can be used to choose between algorithms. 

\emph{Third}, we provide a characterization of the pessimistic algorithm by introducing a weighted-minimax objective. 
In particular, the pessimistic algorithm can be considered to be
 optimal  
in the sense
that it achieves a regret that is comparable to the inherent difficulty of optimal value prediction
on an instance-by-instance basis.
Overall,
the theoretical study we provide consolidates  recent research findings on the impact of being pessimistic in batch policy optimization \citep{BuGeBe20,jin2020pessimism,kumar2020conservative,KiRaNeJo20,yu2020mopo,liu2020provably,yin2021near}. 

The remainder of the paper is organized as follows. After defining the problem setup in Sections \ref{sec:setup}, we present the three main contributions in Sections \ref{sec:minimax} to \ref{sec:pessimism} as aforementioned. Section \ref{sec:related-work} discusses the related works. Section 7 gives our conclusions. 


\section{Problem setup}
\label{sec:setup}

To simplify the exposition, 
we express our results for batch policy optimization 
in the setting of stochastic finite-armed bandits.
In particular, assume the action space consists of $k > 0$ arms, 
where the available data takes the form of \hbox{$n_i>0$} 
real-valued observations $X_{i,1},\dots,X_{i,n_i}$ for each arm $i\in [k]:=\{1,\dots,k\}$.
This data represents the outcomes of $n_i$ pulls of each arm $i$.
We assume further that the data for each arm $i$ is \emph{i.i.d.} with $X_{i,j}\sim P_i$ such that $P_i$ is the reward distribution for arm $i$.
Let $\mu_i=\int x P_i(dx)$ denote the mean reward that results from pulling arm $i$.
All observations in the data set
 $X = (X_{ij})_{i\in [k],j\in [n_i]}$ are assumed to be independent.
 
We consider the problem of designing
an algorithm that takes the counts $(n_i)_{i\in[k]}$ and observations $X\in \times_{i\in [k]} \RR^{n_i}$ as inputs and returns the index of a single arm in $[k]$,
where the goal is to select 
an arm with the highest mean reward.
Let $\cA(X)\in[k]$ be the output of algorithm $\cA$, 
The (simple) regret of $\cA$ is defined as
\begin{align*}
\cR(\cA, \theta) = \mu^* - \EE_{X\sim \theta}[ \mu_{\cA(X)} ]\, ,
\end{align*}
where $\mu^*=\max_i \mu_i$ is the maximum reward. 
Here, the expectation $\EE_{X\sim \theta}$ considers the randomness of the data $X$ generated from problem instance $\theta$, and also any randomness in the algorithm $\cA$, which together induce the distribution of the random choice $\cA(X)$.
Note that this definition of regret depends both on the algorithm $\cA$ and the problem instance $\theta = ((n_i)_{i\in [k]}, (P_i)_{i\in [k]})$.
When $\theta$ is fixed, we will use $\cR(\cA)$ 
to reduce clutter.

For convenience, we also let $n = \sum_i n_i$ and $n_{\min}$ denote the total number of observations and the minimum number of observations in the data.
The optimal arm is $a^*$ and the suboptimality gap is $\Delta_i =  \mu^* - \mu_i$. The largest and smallest non-zero gaps are 
$\Delta_{\max}=\max_i \Delta_i$ and $\Delta_{\min}=\min_{i:\Delta_i>0}\Delta_i$. 
In what follows, we assume that the distributions $P_i$ are \hbox{1-subgaussian} with means 
in the unit interval $[0,1]$.
We denote the set of these distributions by $\cP$. 
The set of all instances where the distributions satisfy these properties is denoted by $\Theta$.
The set of instances with $\bn=(n_i)_{i\in [k]}$ fixed is denoted by $\Theta_{\bn}$.
Thus, $\Theta = \cup_{\bn} \Theta_{\bn}$. 
Finally, we define $|\bn|=\sum_i n_i$ for $\bn=(n_i)_{i\in [k]}$. 

\section{Minimax Analysis}
\label{sec:minimax}

In this section, we introduce the notion of a \emph{confidence-adjusted index algorithm}, and prove that 
a broad range of
such algorithms are minimax optimal up to a logarithmic factor. 
A confidence-adjusted index algorithm is one that calculates an index for each arm based on the data for that arm only, then chooses an arm that maximizes the index. 
We consider index algorithms where the index of arm $i\in [k]$ 
is defined as the sum of the sample mean of this arm, $\hat{\mu}_i = \frac1{n_i} \sum_{j=1}^{n_i} X_{i,j}$ plus a bias term of the form $\alpha/\sqrt{n_i}$ with $\alpha \in \RR$. 
That is, given the input data $X$, the algorithm selects an arm according to
\begin{align*}
\argmax_{i\in[k]}\ \hat{\mu}_i +  \frac{\alpha}{\sqrt{n_i}}  \, .
\addeq\label{eq:index}
\end{align*}
The reason we call these confidence-adjusted 
is because for a given confidence level $\delta>0$, by Hoeffding's inequality, it follows that
\begin{align*}
\mu_i\in\left[ \hat{\mu}_i - \frac{\beta_\delta}{\sqrt{n_i}} , \,\, \hat{\mu}_i + \frac{\beta_\delta}{\sqrt{n_i}} \right]
\addeq\label{eq:confidence-interval}
\end{align*}
with probability at least $1-\delta$ for all arms with 
\begin{align*}
\beta_\delta = \sqrt{2\log\left(\frac{k}{\delta}\right)}\, .
\end{align*}
Thus, the family of confidence-adjusted index algorithms
consists of all algorithms that
follow this strategy, where each particular algorithm is defined by a (data independent) choice of $\alpha$. 
For example, an algorithm specified by $\alpha=-\beta_\delta$ chooses the arm with highest lower-confidence bound (highest LCB value),
while an algorithm specified by $\alpha=\beta_\delta$ chooses the arm with the highest upper-confidence bound (highest UCB value).
Note that $\alpha=0$ corresponds to what is known as the \emph{greedy} (sample mean maximizing) choice.

Readers familiar with the literature
on batch policy optimization will
recognize that $\alpha=-\beta_\delta$ implements what is known as the pessimistic algorithm \citep{jin2020pessimism,BuGeBe20,KiRaNeJo20,yin2021near}, or distributionally robust choice, 
or risk-adverse strategy.
It is therefore natural to question the utility of considering batch policy optimization algorithms that \emph{maximize} UCB values
(i.e., implement optimism in the presence of uncertainty, or risk-seeking behavior, even when there is no opportunity for exploration).
However, our first main result is that for batch policy optimization
a risk-seeking (or greedy) algorithm cannot be distinguished from the more commonly proposed pessimistic approach in terms of minimax regret.

To establish this finding, we first provide a lower bound on the minimax regret:

\begin{theorem}
\label{thm:minmax-lb}
Fix $\bn = (n_i)_{i\in [k]}$ with $n_1 \leq \cdots \leq n_k$.
Then, there exists a universal constant $c > 0$ such that
\begin{align*}
\inf_{\cA} \sup_{\theta\in \Theta_{\bn}} \cR(\cA, \theta) \geq c \max_{m \in [k]} \sqrt{\frac{\max(1, \log(m))}{n_m}} \,.
\end{align*}
\end{theorem}

The assumption of increasing counts, $n_1 \leq \cdots \leq n_k$, is only needed to simplify the statement;
the arm indices can always be re-ordered without loss of generality.
The proof follows by arguing that the minimax regret is lower bounded by the Bayesian regret of the Bayesian optimal policy for any prior.
Then, with a judicious choice of prior, the Bayesian optimal policy has a simple form.
Intuitively, the available data permits estimation of the mean of action $a$ with accuracy $O(\sqrt{1/n_a})$.
The additional logarithmic factor appears when $n_1,\ldots,n_m$ are relatively close, in which case the lower bound is demonstrating the necessity of
a union bound that appears in the upper bound that follows.
The full proof appears in the supplementary material.

Next we show that a wide range of confidence-adjusted index algorithms are nearly minimax optimal when their confidence parameter is properly chosen:
\begin{theorem}
Fix $\bn = (n_i)_{i\in [k]}$. 
Let $\delta$ be the solution of $\delta = \sqrt{32 \log(k/\delta)/\min_i n_i}$, and
$\cI$ be the confidence-adjusted index algorithm with parameter $\alpha$. 
Then, for any $\alpha\in [-\beta_\delta,\beta_\delta]$, we have
\begin{align*}
\sup_{\theta\in \Theta_{\bn}} \cR(\cI(\alpha),\theta) \le 
12 \sqrt{\frac{\log(k/\delta)}{\min_i n_i}}\,.
\end{align*}
\label{thm:minimax-upper}
\end{theorem}
\begin{remark}
Theorem~\ref{thm:minimax-upper} also holds for algorithms that use different $\alpha_i\in [-\beta_\delta,\beta_\delta]$
for different arms.
\end{remark}

Perhaps a little unexpectedly, we see that \emph{regardless} of optimism vs.\ pessimism, index algorithms with the right amount of adjustment, or \emph{even no adjustment}, are minimax optimal, up to an order $\sqrt{\log(k n)}$ factor. 
We note that  although these algorithms have the same worst case performance, they can behave very differently indeed on individual instances, as we show in the next section. 

In effect, what these two results tell us is that minimax optimality is too weak as a criterion to distinguish between pessimistic versus optimistic (or greedy) algorithms when considering the ``fixed count'' setting of batch policy optimization.
This leads us to ask whether more refined optimality criteria are able to provide nontrivial guidance in the selection of batch policy optimization methods.
One such criterion, considered next, is known as instance-optimality in the literature of cumulative regret minimization for stochastic bandits.

\section{Instance-Dependent Analysis}
\label{sec:instance}

To better distinguish between algorithms
we require a much more refined notion of performance
that goes beyond merely considering worst-case
behavior over all problem instances.
Even if two algorithms have the same worst case performance, they can behave very differently on individual instances.
Therefore, we consider the instance dependent performance 
of confidence-adjusted index algorithms.

\subsection{Instance-dependent Upper Bound} 
 
Our next result provides a regret upper bound for a general form of index algorithm. 
All upper bounds in this section hold for any $\theta \in \Theta_{\bn}$ unless otherwise
specified, and we use $\cR(\cA)$ instead of $\cR(\cA, \theta)$ to simplify the notation.

\begin{theorem}
Consider a general form of index algorithm, $\cA(X)=\argmax_i \hat{\mu}_i + b_i$, where $b_i$ denotes the bias for arm $i\in[k]$ specified by the algorithm. 
For $2\le i\le k$ and $\eta\in \RR$, define
\begin{align*}
g_i(\eta) =  \sum_{j\ge i} e^{-\frac{n_j}{2} \left(\eta - \mu_j - b_j\right)_+ ^2}
+ \min_{j < i} e^{-\frac{n_j}{2} \left(\mu_j + b_j - \eta \right)_+ ^2}
\end{align*}
and $g_i^* = \min_{\eta} g_i(\eta)$. Assuming $\mu_1\ge \mu_2\ge \cdots \ge\mu_k$, for the index algorithms (\ref{eq:index}) we have
\begin{align*}
\Prb{\cA(X)\ge i} \le \min\{ 1, g_i^* \}  \addeq\label{eq:ub-cdf-general}
\end{align*}
and
\begin{align*}
\cR(\cA) \le \sum_{2\le i \le k} \Delta_i \left(\min\{ 1, g_i^* \} - \min\{ 1, g_{i+1}^* \} \right) \addeq\label{eq:ub-exp-general}
\end{align*}
where we define $g_{k+1}^*=0$.
\label{thm:instance-upper-general}
\end{theorem}

The assumption $\mu_1 \geq \mu_2 \geq \dots \geq \mu_k$ is only required to express the statement simply;
the indices can be reordered without loss of generality.
The expression in \eqref{eq:ub-cdf-general} is a bit difficult to work with, so to make the subsequent analysis simpler we provide a looser but more interpretable bound for general index algorithms  as follows.
\begin{corollary}
Following the setting of Theorem~\ref{thm:instance-upper-general}, consider any index algorithm 
and any $\delta\in(0, 1)$. Define $U_i=\mu_i + b_i + \beta_{\delta}/\sqrt{n_i}$ and $L_i = \mu_i + b_i - \beta_{\delta}/\sqrt{n_i}$. 
Let $h = \max\{i\in[k]: \max_{j<i} L_j < \max_{j'\ge i} U_{j'}\}$. Then we have
\begin{align*}
& \cR(\cA) \le  \Delta_h + \frac{\delta}{k}\Delta_{\max} \\
& + \frac{\delta}{k} \sum_{i>h}(\Delta_i - \Delta_{i-1})
\sum_{j\ge i} e^{-\frac{n_j}{2}\left( \max_{j'<i}L_{j'} - U_j \right)^2 } \,.
\end{align*}
\label{coro:instance-upper-general-simplified}
\end{corollary}

\begin{remark}
The upper bound in Corollary~\ref{coro:instance-upper-general-simplified} can be further relaxed as 
$\cR(\cA) \le  \Delta_h + \delta \Delta_{\max}$.
\end{remark}

\begin{remark}
\label{remark:recover-minimax}
The minimax regret upper bound (Theorem~\ref{thm:minimax-upper}) can be recovered a result of Corollary~\ref{coro:instance-upper-general-simplified} 
(see supplement).
\end{remark}


Corollary~\ref{coro:instance-upper-general-simplified} highlights an inherent optimization property of index algorithms: they work by designing an additive adjustment for each arm, such that all of the bad arms ($i>h$) can be eliminated efficiently, i.e., it is desirable to make $h$ as small as possible. 
We note that although one can directly plug in the specific choices of $\{b_i\}_{i\in \iset{k}}$
to get instance-dependent upper bounds for different algorithms, it is not clear how their performance compares to one another.
Therefore, we 
provide simpler relaxed
upper bounds for the three specific cases,
greedy, LCB and UCB, to allow us to better differentiate their performance across different problem instances
(see supplement for details).  


\begin{corollary}[Regret Upper bound for Greedy]
Following the setting of Theorem~\ref{thm:instance-upper-general},  
for any $0<\delta<1$, the regret of greedy ($\alpha=0$) on any problem instance is upper bounded by 
\begin{align*}
\cR(\cA) \le  \min_{i\in \iset{k}} \left(\Delta_{i} + \sqrt{\frac{2}{n_{i}}\log\frac{k}{\delta}} + \max_{j>i} \sqrt{\frac{2}{n_{j}}\log\frac{k}{\delta}} \right) + \delta \, .
\end{align*}
\label{coro:ub-greedy-instance}
\end{corollary}

\begin{corollary}[Regret Upper bound for LCB]
Following the setting of Theorem~\ref{thm:instance-upper-general}, 
for any $0<\delta<1$, the regret of LCB ($\alpha=-\beta_\delta$) on any problem instance is upper bounded by
\begin{align*}
\cR(\cA)\leq \min_{i\in \iset{k}} \Delta_i + \sqrt{\frac{8}{n_{i}}\log\frac{k}{\delta}} + \delta\, .
\end{align*}
\label{coro:ub-lcb-instance}
\end{corollary}

\begin{corollary}[Regret Upper bound for UCB]
Following the setting of Theorem~\ref{thm:instance-upper-general},
for any $0<\delta<1$, the regret of UCB ($\alpha=\beta_\delta$) on any problem instance is upper bounded by
\begin{align*}
\cR(\cA)  \le \min_{i \in \iset{k}} \left( \Delta_i + \max_{j > i} \sqrt{\frac{8}{n_{j}}\log\frac{k}{\delta}} \right) + \delta \, .
\end{align*}
\label{coro:ub-ucb-instance}
\end{corollary}

\begin{remark}
The results 
in these corollaries sacrifice
the tightness of instance-dependence to obtain cleaner bounds for the different algorithms. The tightest instance dependent bounds can be derived from Theorem~\ref{thm:instance-upper-general} by 
optimizing
$\eta$. 
\end{remark}

\paragraph{Discussion.}
The regret upper bounds presented above suggest that although they are all nearly minimax optimal, UCB, LCB and greedy exhibit 
distinct behavior on individual instances.
Each
will eventually select the best arm with high probability
when $n_i$ gets large for \emph{all} $i\in \iset{k}$, 
but their performance 
can be
very different when $n_i$ gets large for only a \emph{subset} of arms $S\subset \iset{k}$. 
For example,
LCB performs well whenever $S$ contains a good arm 
(i.e., 
with small $\Delta_i$ and large $n_i$). 
UCB performs well
when 
there 
is
a good arm $i$ 
such that all worse arms
are in $S$ ($n_j$ 
large for all $j>i$). 
For the greedy algorithm, the regret upper bound is small only when 
there is a good arm $i$ where $n_j$ is large for all $j \ge i$, in which situation both LCB and UCB 
perform well.

Clearly there are
instances where LCB performs much better than UCB 
and vice versa.
Consider an environment where there are two groups of arms: one 
with
higher rewards and 
another with lower rewards.
The behavior policy 
plays a subset of the arms $S\subset\iset{k}$ 
a large number of times
and ignores the rest. 
If $S$ contains at least one good arm but no bad arm,
LCB will select a good played arm~(with high probability)
while UCB will select a bad unplayed arm.
If $S$ 
consists
of all bad arms,
then LCB will select a bad arm by being pessimistic about the unobserved good arms
while UCB is guaranteed to select a good arm by being optimistic. 

This example actually
raises a potential reason 
to favor LCB, since
the condition for UCB to outperform LCB is stricter: requiring the behavior policy to play all bad arms
while ignoring all good arms. 
To formalize this, we compare the upper bounds for the two algorithms by taking
the $n_i$ for a subset of arms $i\in S\subset \iset{k}$ to infinity.
For $\cA\in \{\greedyalg, \lcbalg, \ucbalg\}$, 
let $\hat{\cR}_S(\cA)$ be the regret upper bounds with $\{n_i\}_{i\in S}\to\infty$ and $\{n_i\}_{i\notin S} = 1$ while
fixing $\mu_1,...,\mu_k$  in
Corollary~\ref{coro:ub-greedy-instance}, \ref{coro:ub-lcb-instance},  and \ref{coro:ub-ucb-instance} respectively. 
Then 
LCB 
dominates
the three algorithms with high probability 
under a uniform prior for $S$:
\begin{proposition}
Suppose $\mu_1>\mu_2>...>\mu_k$ and $S\subset \iset{k}$ is uniformly sampled from all subsets with size $m<k$,
then 
\begin{align*}
\Prb{\hat{\cR}_S(\lcbalg) < \hat{\cR}_S(\ucbalg)} \ge 1 - 
\frac{(k-m)!m!}{k!} \,.
\end{align*}
\label{prop:lcb-vs-ucb}
\end{proposition}
This lower bound
is $1/2$ when $k=2$ and approaches $1$ when $k$ increases for any $0 < m < k$ since it is always lower bounded by $1 - 1/k$.
The same argument applies when comparing LCB to greedy.

To summarize, when comparing different algorithms by their upper bounds, we have the following 
observations:
(i) These algorithms behave differently on different instances, and none of them outperforms
the others on all instances.
(ii) Both scenarios where LCB is better and scenarios where UCB is better exist.
(iii) LCB is more favorable when $k$ is not too small because it is the best option
among these algorithms on most of the instances.

\paragraph{Simulation results. }
Since our discussion is based on comparing only the upper bounds~(instead of the exact regret) for different algorithms,
it is a question that whether these statements still hold in terms of their actual performance.
To answer this question, we verify these statements through experiments on synthetic problems.
The details of these synthetic experiments can be found in the supplementary material.

We first verify that there exist instances where LCB is the best among the three algorithms as well as instances
where UCB is the best.
For LCB to perform well, we construct two $\epsilon$-greedy behavior policies on a $100$-arm bandit where the best arm or a near-optimal arm is selected 
to be played with a high frequency while the other arms are uniformly played with a low frequency.
Figure~\ref{fig:lcb-1} and~\ref{fig:lcb-2} show that LCB outperforms UCB and greedy on these two instances,
verifying our observation from the upper bound~(Corollary~\ref{coro:ub-lcb-instance}) that LCB only requires a good behavior policy 
while UCB and greedy require bad arms to be eliminated~(which is not the case for $\epsilon$-greedy policies).
For UCB to outperform LCB, we set the behavior policy to play 
a set of near-optimal arms with only a small number of times and play the rest of the arms uniformly. 
Figure~\ref{fig:ucb-1} and~\ref{fig:ucb-2} show that UCB outperforms LCB and greedy on these two instances, verifying our observation from the upper bound~(Corollary~\ref{coro:ub-ucb-instance}) that UCB only requires all worse arms to be identified.

We now verify the statement that LCB is the best option on most of the instances when $k$ is not too small. We verify this statement in two aspects: First, we show that when $k=2$, LCB and UCB have an equal 
chance to be the better algorithm. More specifically, we fix $n_1 > n_2$~(note that if $n_1=n_2$ all index algorithms are the same as greedy) and vary $\mu_1 - \mu_2$ from $-1$ to $1$. 
Intuitively, when $|\mu_1-\mu_2|$ is large, the problem is relatively easy for all algorithms. For $\mu_1-\mu_2$ in the medium range, as it becomes larger, the good arm is tried more often, thus the problem becomes easier for LCB and harder for UCB. 
Figure~\ref{fig:two-arm-1} and ~\ref{fig:two-arm-2} 
confirm this and 
show that both LCB and UCB are the best option on half of the instances. Second, 
we show that as $k$ grows, LCB quickly becomes the more favorable algorithm,
outperforming UCB and greedy on an increasing fraction of instances.
More specifically, we vary $k$ and sample a set of instances from the prior distribution introduced in 
Proposition~\ref{prop:lcb-vs-ucb} with $|S|=k/2$ and $|S|=k/4$. Figure~\ref{fig:frac-1} and~\ref{fig:frac-2} shows that the fraction of instances where LCB is the best quickly approaches $1$ as $k$ increases.

\begin{figure*}[t]
\centering
\subfigure[LCB-1]
{\includegraphics[width=6.0cm]{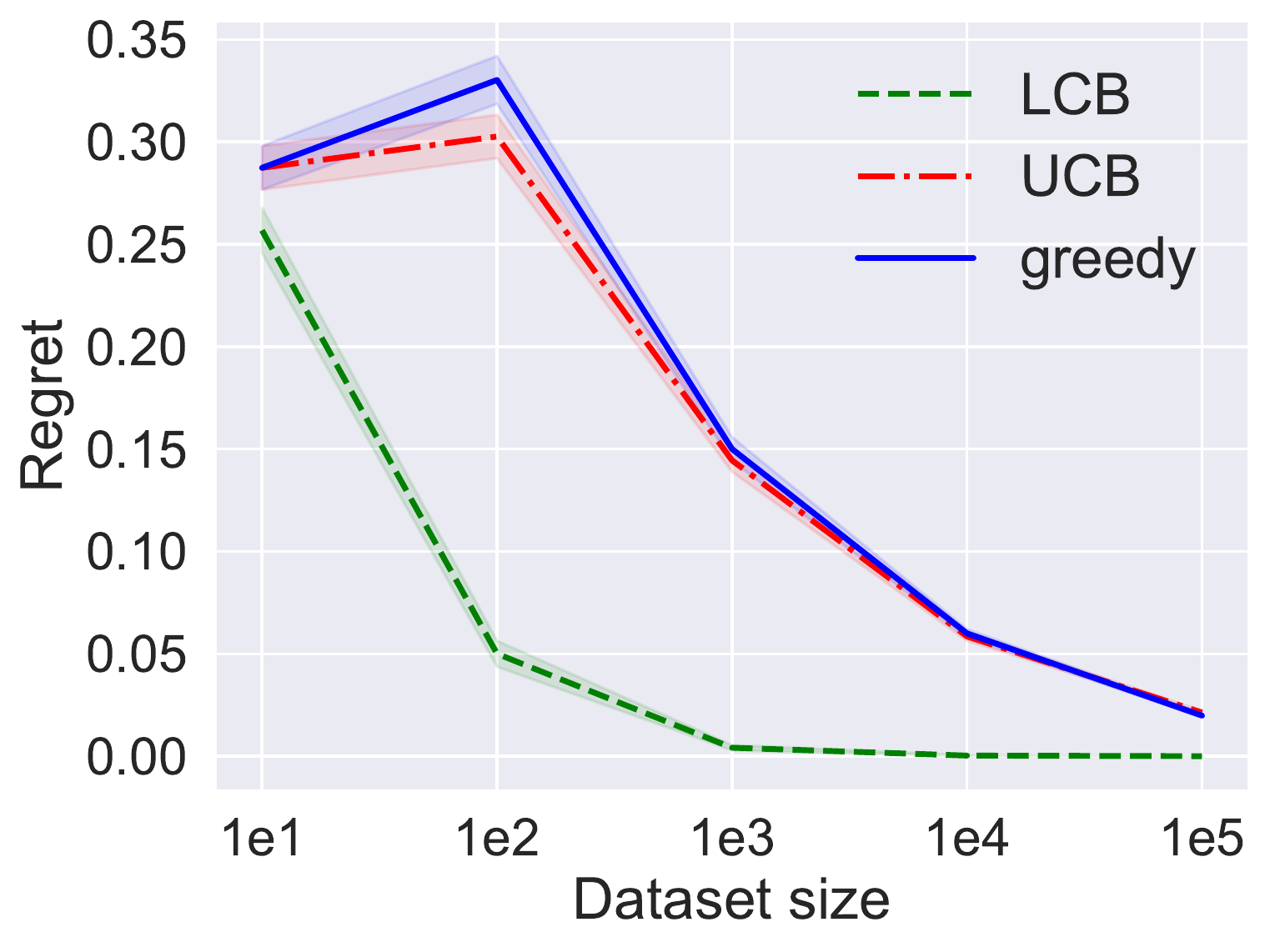}
\label{fig:lcb-1}
}
\subfigure[LCB-2]
{\includegraphics[width=6.0cm]{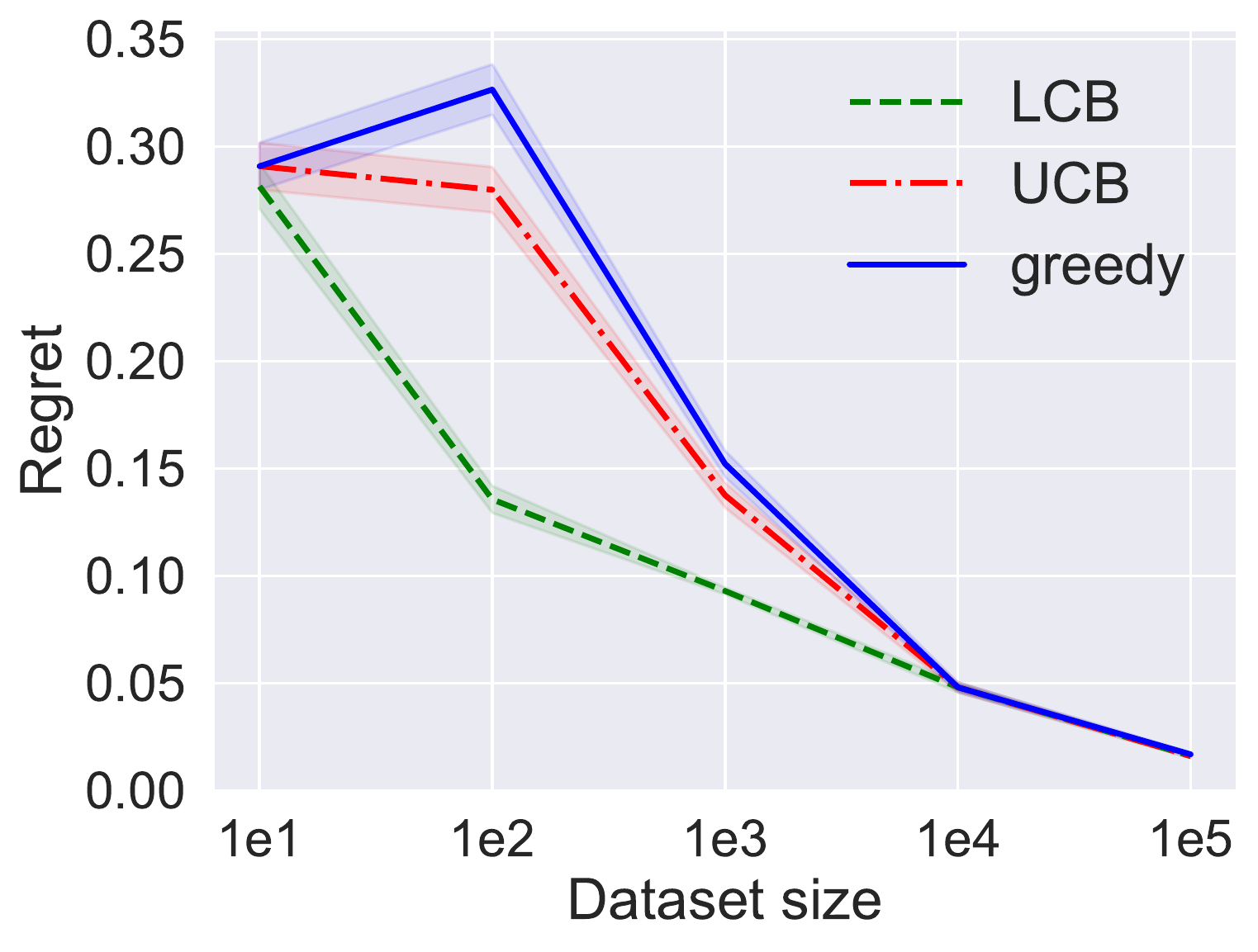}
\label{fig:lcb-2}
}
\subfigure[UCB-1]
{\includegraphics[width=6.0cm]{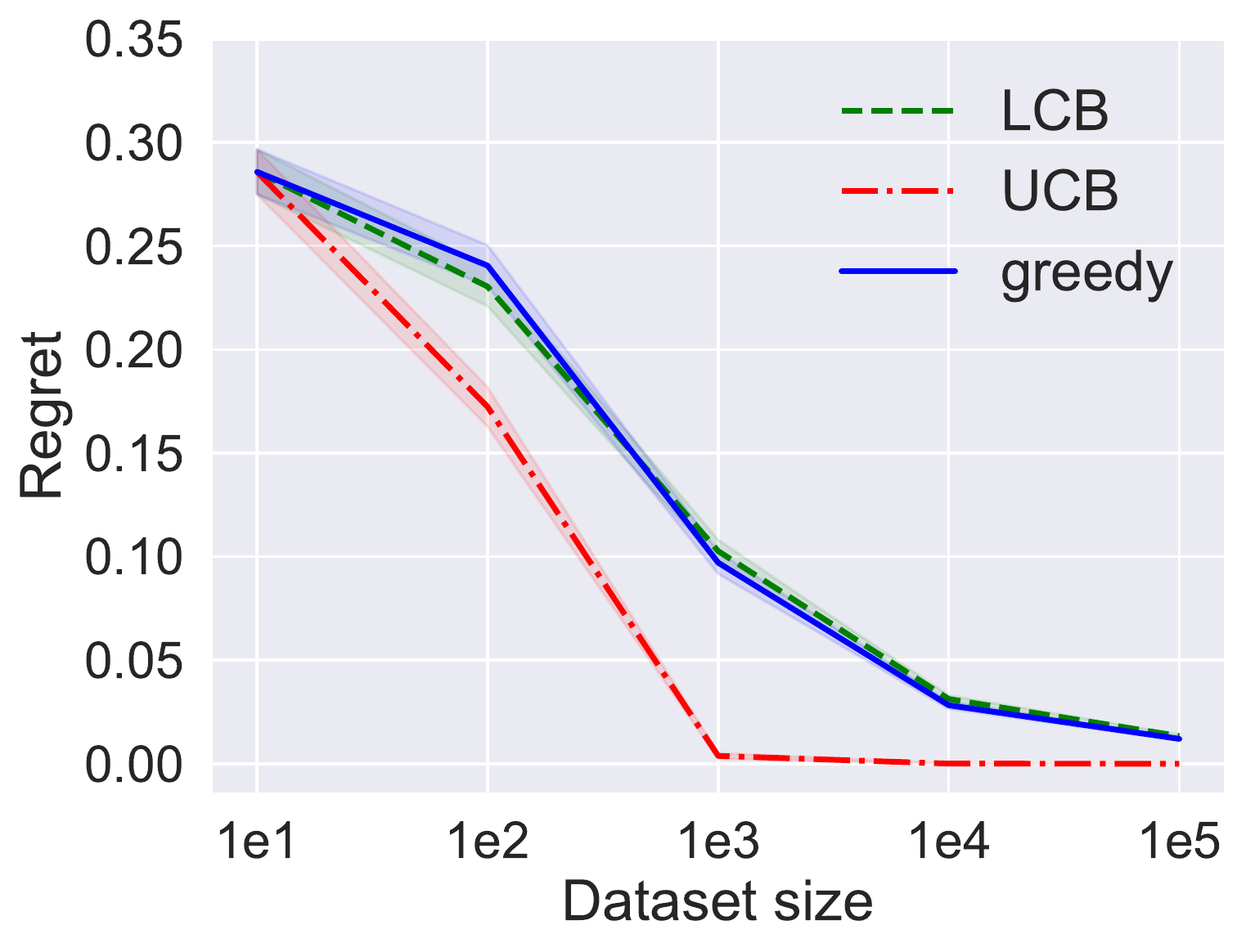}
\label{fig:ucb-1}
}
\subfigure[UCB-2]
{\includegraphics[width=6.0cm]{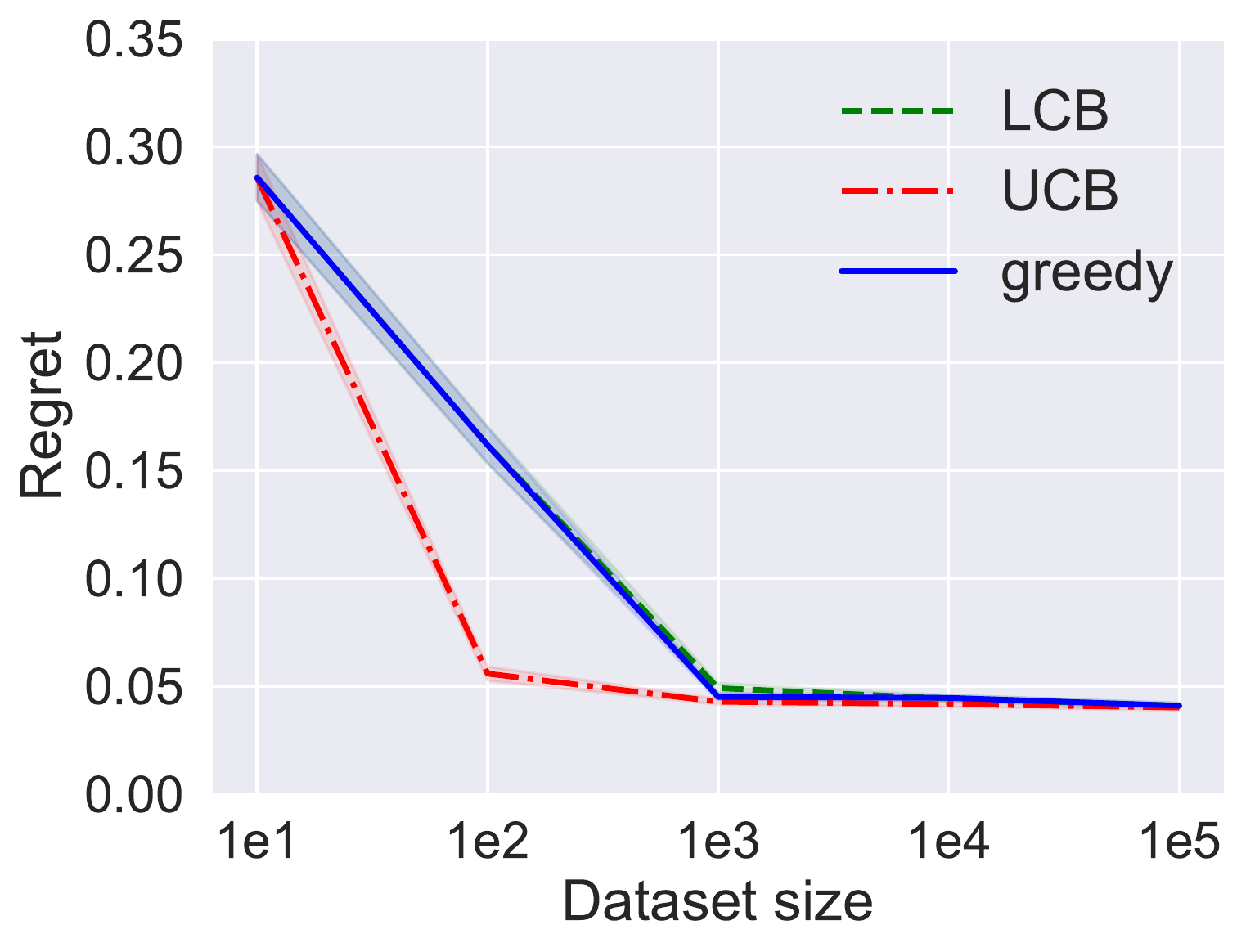}
\label{fig:ucb-2}
}
\caption{Comparing UCB, LCB and greedy on synthetic problems~(with $k=100$). 
(a) and (b): Problem instances where LCB has the best performance. The data set is generated by a behavior policy that pulls an arm~$i$ with \emph{high} frequency and the other arms uniformly. In (a) $i$ is the best arm
while in (b) $i$ is the $10$th-best arm.
(c) and (d): Problem instances where UCB has the best performance. The data set is generated by a behavior policy that pulls a set of good arms $\{j: j\le i\}$ 
with very \emph{small} frequency and the other arms uniformly. In (c) we use $i=1$ while in (d) we use $i=10$.  
Experiment details are provided in the supplementary material. 
}
\label{fig:bandit}
\end{figure*}
\begin{figure*}[ht]
\centering
\subfigure[$k=2, n_1/n_2=2$]
{\includegraphics[width=6.0cm]{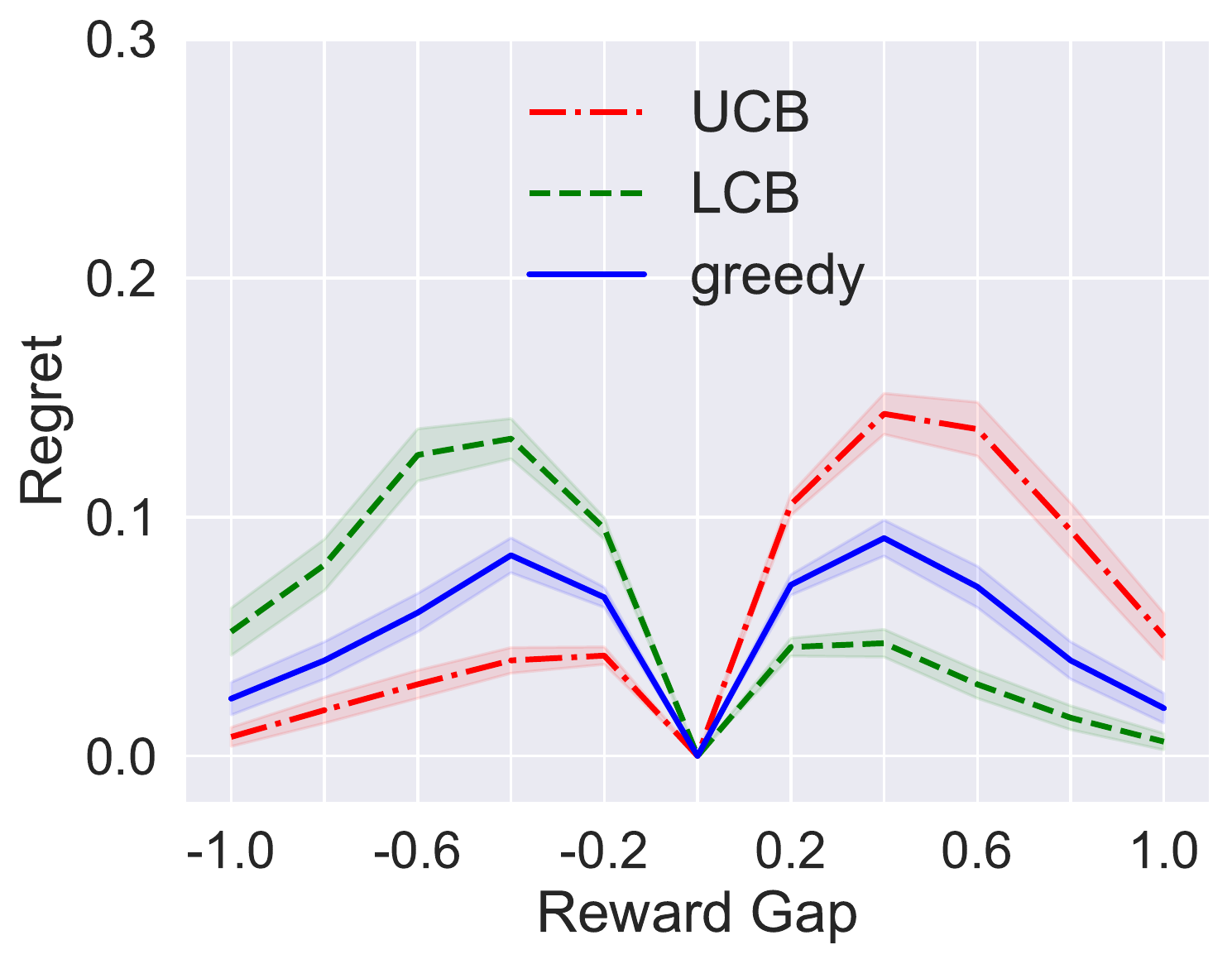}
\label{fig:two-arm-1}
}
\subfigure[$k=2, n_1/n_2=10$]
{\includegraphics[width=6.0cm]{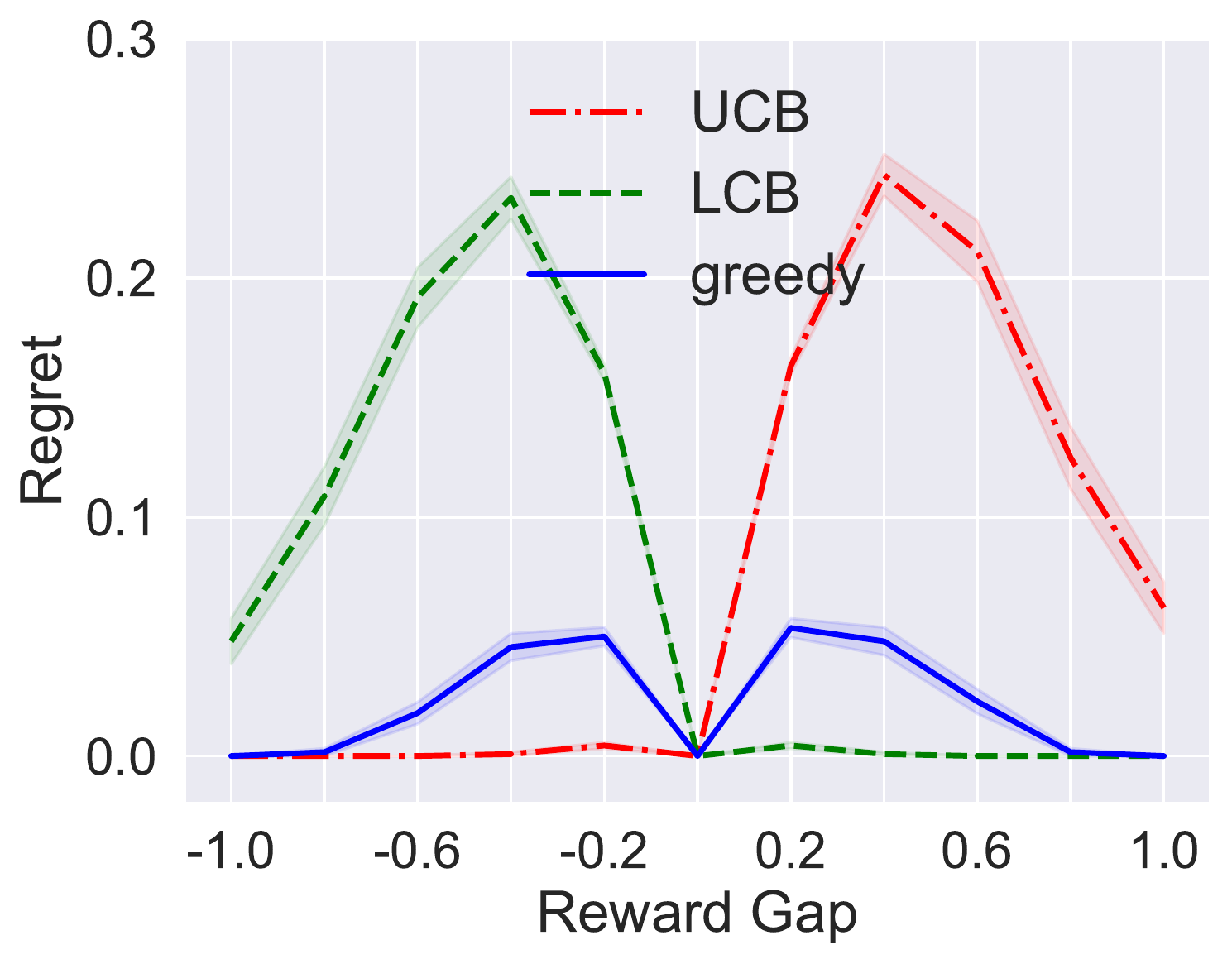}
\label{fig:two-arm-2}
}
\subfigure[$|S|=k/2$]
{\includegraphics[width=6.0cm]{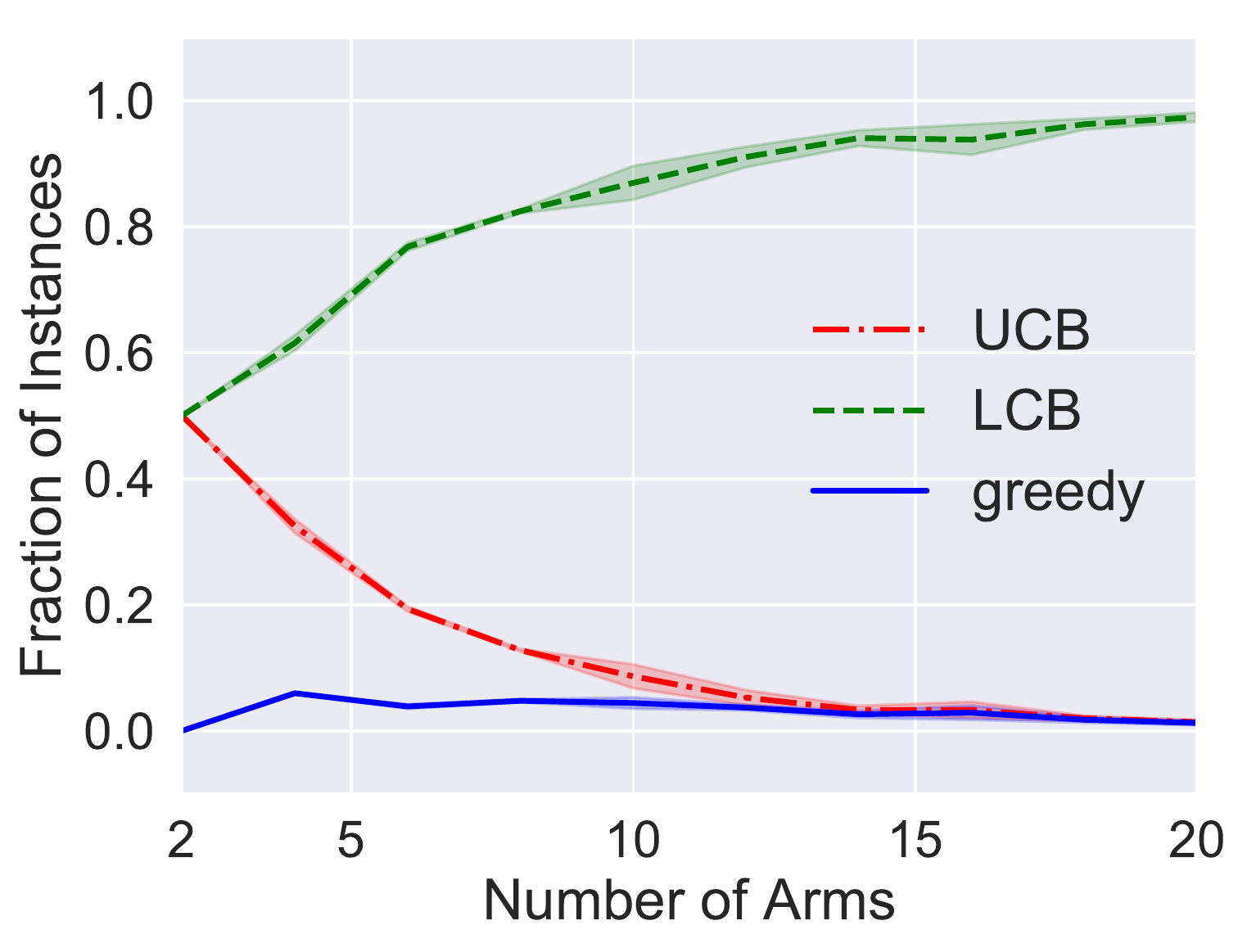}
\label{fig:frac-1}
}
\subfigure[$|S|=k/4$]
{\includegraphics[width=6.0cm]{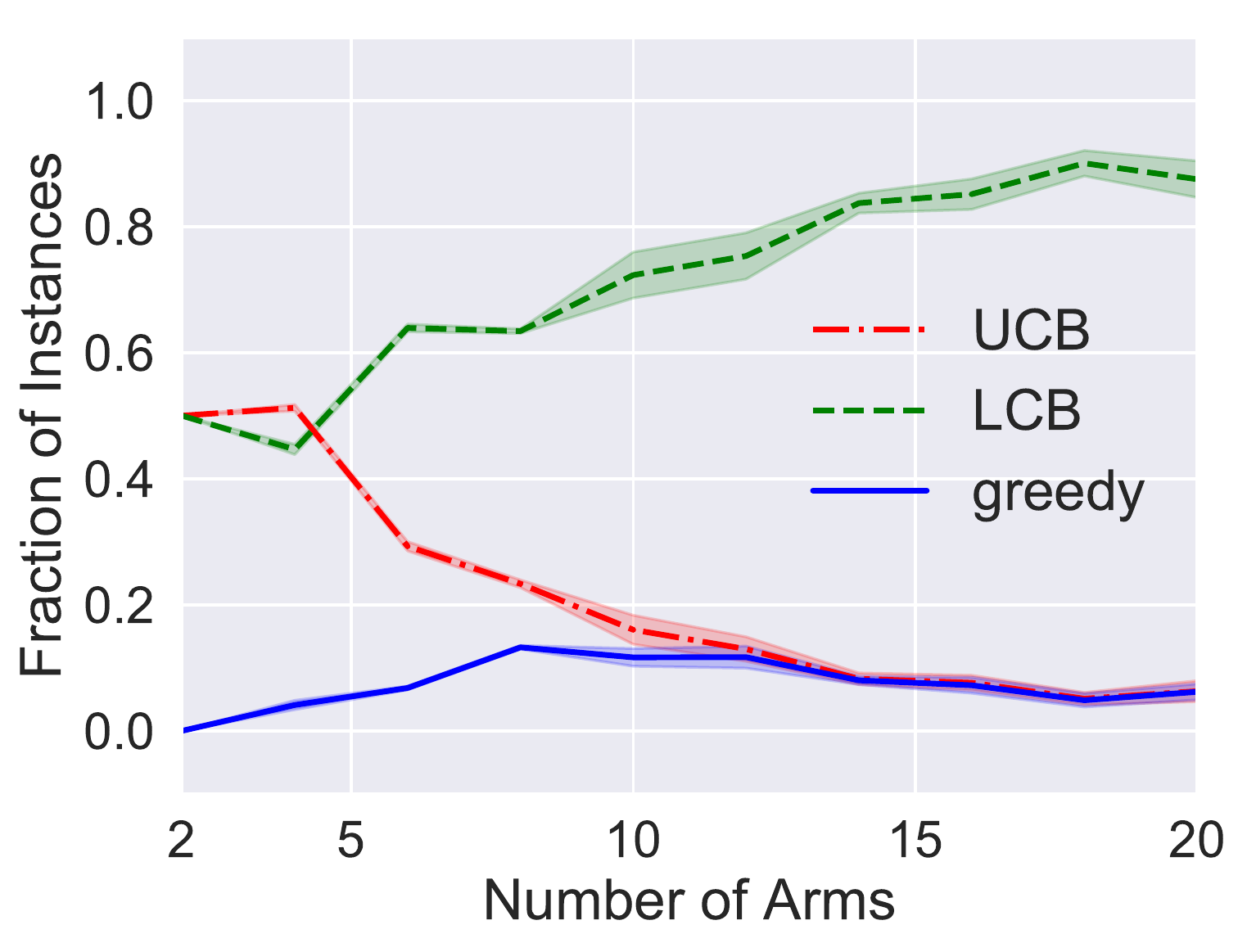}
\label{fig:frac-2}
}
\caption{Comparing UCB, LCB and greedy on synthetic problems. 
(a) and (b): A set of two-armed bandit instances where both LCB and UCB dominate half of the instances.  
(c) and (d): 
For each $k$,  we first sample 100 vectors $\vec{\mu}=[\mu_1,...,\mu_k]$ and for each $\vec{\mu}$
we uniformly sample $100$~(if exist) subsets $S\subset{k}, |S|=m$~($m=k/2$ in (c) and $m=k/4$ in (d)), to generate up to $10$k instances.
We then count the fraction of instances where each algorithm performs better than the other two algorithms
among the randomly sampled set of instances. 
Experiment details are provided in the supplementary material. 
}
\label{fig:two-armed-bandit}
\end{figure*}

\subsection{Instance-dependent Lower Bound}

We have established that, despite all being minimax optimal, index algorithms with different adjustment can exhibit very different performance on specific problem instances.
One might therefore wonder if instance optimal algorithms exist for batch policy optimization with finite-armed stochastic bandits. 
To answer this question, we next show that there is no instance optimal algorithm in the batch optimization setting for stochastic bandits, which is a very different outcome from
the setting of cumulative regret minimization for online stochastic bandits. 

For cumulative regret minimization, \citet{lai1985asymptotically} introduced an asymptotic notion of instance optimality \citep{lattimore2020bandit}. The idea is to first remove algorithms that are insufficiently adaptive, then define a yardstick (or benchmark) for each instance as the best (normalized) asymptotic performance that can be achieved with the remaining adaptive algorithms.
An algorithm that meets this benchmark over all instances is then considered to be an instance optimal algorithm. 

When 
adapting this notion of instance optimality to the batch setting
there are two decisions that need to be made: what is an appropriate notion of ``sufficient adaptivity'' and whether, of course, a similar asymptotic notion is sought or optimality can be adapted to the finite sample setting.
Here, we consider the asymptotic case, as one usually expects this  to be easier.

We consider the 2-armed bandit case ($k=2$) with Gaussian reward distributions $\cN(\mu_1, 1)$ and $\cN(\mu_2, 1)$ for each arm respectively. Recall that, in this setting, fixing $\bn=(n_1, n_2)$ each instance $\theta\in \Theta_{\bn}$ is defined by $(\mu_1, \mu_2)$. We assume that algorithms
only make decisions based on the sufficient statistic --- empirical means for each arm, which in this case reduces to $X=(X_1, X_2, \bn)$ with $X_i\sim \cN(\mu_i, 1/n_i)$.

To introduce an asymptotic notion,
we further denote $n=n_1 + n_2$, $\pi_1 = n_1/n$, and $\pi_2=n_2/n=1 - \pi_1$. 
Assume $\pi_1, \pi_2 >0$; then each $\bn$ can be uniquely defined by $(n, \pi_1)$ for $\pi_1\in (0, 1)$.
We also ignore the fact that $n_1$ and $n_2$ should be integers since we assume the algorithms can
only make decisions based on the sufficient statistic $X_i\sim \cN(\mu_i, 1/n_i)$, which is well defined even when $n_i$ is not an integer.

\begin{definition}[Minimax Optimality]
Given a constant $c\ge 1$, an algorithm is said to be minimax optimal if its worst case regret
is bounded by the minimax value of the problem up to a multiplicative factor $c$. We define the set 
of minimax optimal algorithms as
\begin{align*}
\cM_{\bn, c} = \left\{ \cA: \sup_{\theta\in \Theta_{\bn}}\cR(\cA, \theta) \le c\cdot \inf_{\cA'} \sup_{\theta\in \Theta_{\bn}}\cR(\cA', \theta) \right\}\,.
\end{align*}
\end{definition}
\begin{definition}[Instance-dependent Lower Bound]
Given a set of algorithms $\cM$, for each $\theta \in \Theta_{\bn}$, we define the instance-dependent lower
bound as $\cR^*_{\cM}(\theta)=\inf_{\cA\in \cM} \cR(\cA, \theta)$.
\end{definition}

The following theorem states the non-existence of instance optimal algorithms up to a constant multiplicative factor.
\begin{theorem}
Let $c_0$ be the constant in  minimax lower bound such that $\inf_{\cA} \sup_{\theta\in \Theta_{\bn}}\cR(\cA, \theta) \ge c_0/\sqrt{n_{\min}}$. Then for any $c > 2/c_0$ and any algorithm $\cA$, we have
\begin{align*}
\sup_{\theta\in \Theta_{\bn}} \frac{\cR(\cA, \theta)}{\cR^*_{\cM_{\bn, c}}(\theta)}
\ge \frac{ n_{\min} }{ n_{\min} + 4 } e^{\frac{\beta^2}{4} +  \frac{\beta}{4}\sqrt{n_{\min}}  }
\end{align*}
where $\beta=cc_0 - 2$.
\label{thm:instance-negative-1}
\end{theorem}

\begin{corollary}
There is no algorithm that is instance optimal up to a constant multiplicative factor. That is,
fixing $\pi_1\in (0, 1)$,
given any $c>2/c_0$ and for any algorithm $\cA$
, we have
\begin{align*}
\limsup_{n\to\infty} \sup_{\theta\in \Theta_{\bn}} \frac{\cR(\cA, \theta)}{\cR^*_{\cM_{\bn, c}}(\theta)}
= +\infty \,.
\end{align*}
\label{cor:instance-negative}
\vspace{-3mm}
\end{corollary}

The proof of Theorem \ref{thm:instance-negative-1} 
follows by constructing two competing instances where the performance of any single algorithm 
cannot simultaneously match the performance of the adapted algorithm on each specific instance. 
Here we briefly discuss the proof idea -- the detailed analysis is provided in the supplementary material. 

\emph{Step 1}, define the algorithm 
$\cA_{\beta}$ as 
\begin{align*}
\cA_{\beta}(X) = 
\begin{cases}
1 & \textrm{ if } X_1 - X_2 \ge \frac{\beta}{\sqrt{n_{\min}}} \\
2 & \textrm{ otherwise }
\end{cases}\,.
\end{align*}
For any $\beta$ within a certain range, it can be shown that $\cA_{\beta} \in \cM_{\bn, c}$, hence $\cR^*_{\cM_{\bn, c}}(\theta)\leq \cR(\cA_\beta, \theta)$. 

\emph{Step 2}, 
construct two problem instances as follows.  
Fix a $\lambda \in \mathbb{R}$ and $\eta > 0$, and define
$$\theta_1=(\mu_1, \mu_2)=(\lambda + \frac{\eta}{n_1}, \lambda - \frac{\eta}{n_2})\,,$$ $$\theta_2=(\mu'_1, \mu'_2)=(\lambda - \frac{\eta}{n_1}, \lambda + \frac{\eta}{n_2})\,.$$
Since we have
$X_1-X_2\sim \cN(\Delta,\sigma^2 )$ 
on instance $\theta_1$ and
$X_1-X_2\sim \cN(-\Delta,  \sigma^2 )$
on instance $\theta_2$, 
where $\Delta=( \frac{1}{n_1} + \frac{1}{n_2} )\eta$ and $\sigma^2=\frac{1}{n_1} + \frac{1}{n_2}$, 
the regret of $\cA_\beta$ on both instances can be computed using the CDF of Gaussian distributions. 
Note that $\cR(\cA_{-\beta}, \theta_1) =\cR(\cA_{\beta}, \theta_2)$. 
We now chose a $\beta_1<0$ for $\theta_1$ to upper bound  $\cR^*_{\cM_{\bn, c}}(\theta_1)$ by $\cR(\cA_{\beta_1}, \theta_1)$ and use $\beta_2 = -\beta_1 > 0$ to upper bound $\cR^*_{\cM_{\bn, c}}(\theta_2)$ by $\cR(\cA_{\beta_2}, \theta_1)$.

Then applying the Neyman-Pearson Lemma \citep{neyman1933ix} to this scenario gives that $\cA_0$
is the optimal algorithm in terms of balancing the regret on $\theta_1$ and $\theta_2$:
\begin{align*}
\cR(\cA_0, \theta_1)= \cR(\cA_0, \theta_2) = \min_{\cA} \max\{\cR(\cA, \theta_1), \cR(\cA, \theta_2)\} \,.
\end{align*}

\vspace{2mm}
\emph{Step 3}, combining the above results gives
\begin{align*}
\sup_{\theta\in \Theta_{\bn}} \frac{\cR(\cA, \theta)}{\cR^*_{\cM_{\bn, c}}(\theta)} 
 &\ge  \max \left\{ \frac{\cR(\cA, \theta_1)}{\cR^*_{\cM_{\bn, c}}(\theta_1)}, \frac{\cR(\cA, \theta_2)}{\cR^*_{\cM_{\bn, c}}(\theta_2)}\right\} \\
& \ge \max \left\{ \frac{\cR(\cA, \theta_1)}{\cR(\cA_{\beta_1}, \theta_1)}, \frac{\cR(\cA, \theta_2)}{\cR(\cA_{\beta_2}, \theta_2)}\right\} \\
& = \frac{\max\left\{  \cR(\cA, \theta_1), \cR(\cA, \theta_2)\right\}}{\cR(\cA_{\beta_1}, \theta_1)} \\
& \ge \frac{\cR(\cA_0, \theta_1)}{\cR(\cA_{\beta_1}, \theta_1)} 
\,.
\end{align*}

Note that both the regret $\cR(\cA_0, \theta_1)$ and $\cR(\cA_{\beta_1}, \theta_1)$ can be exact expressed 
as CDFs of Gaussian distributions: $\cR(\cA_0, \theta_1)=\Phi\left( - \Delta/\sigma \right)$ and $\cR(\cA_{\beta_1}, \theta_1)=\Phi\left( - \beta/(\sigma \sqrt{n_{\min}})
- \Delta/\sigma \right)$
where $\Phi$ is the CDF of the standard normal distribution.

Now we can conclude the proof by picking $\lambda=1/2$ and $\eta=n_{\min}/2$ such that $\theta_1, \theta_2 \in [0, 1]^2$.
Then the result in Theorem \ref{thm:instance-negative-1} 
can be proved by applying an approximation of $\Phi$ and setting $\beta_1=-\beta_2=2-cc_0$ such that both $\beta_1$ and $\beta_2$ are within the range that makes $\cA_{\beta}\in \cM_{\bn, c}$. 

To summarize, for any algorithm that performs well on some problem instance, there exists another instance where the same algorithm suffers arbitrarily larger regret. 
Therefore, any reasonable algorithm is equally optimal, or not optimal, depending on 
whether the minimax or instance optimality is considered. In this sense, 
there remains a lack of a well-defined optimality criterion that can be used to choose between algorithms
for batch policy optimization.

\vspace{-2mm}
\section{A Characterization of Pessimism}
\label{sec:pessimism}
\vspace{-1mm}

It is known that the pessimistic algorithm, maximizing a lower confidence bound on the value, satisfies many desirable properties:
it is consistent with rational decision making using preferences that satisfy uncertainty aversion and certainty-independence \citep{GiSch89},
it avoids the optimizer's curse \citep{SmiWi06},
it allows for optimal inference in an asymptotic sense 
\citep{Lam2019},
and in a 
certain sense it is the unique strategy that achieves these properties \citep{VPEKuhn17,SuVPKuhn20}.
However, a pure statistical decision theoretic justification (in the sense of
\citet{Berger85})
is still lacking.

The instance-dependent lower bound 
presented above attempts
to characterize the optimal performance of an algorithm on an instance-by-instance basis.
In particular, one can interpret the objective ${\cR(\cA, \theta)}/{\cR^*_{\cM_{\bn, c}}(\theta)}$ defined in Theorem~\ref{thm:instance-negative-1} as weighting each instance $\theta$ by $1/{\cR^*_{\cM_{\bn, c}}(\theta)}$,
where this can be interpreted as a measure of instance difficulty.
It is natural to consider an algorithm to be optimal if it can perform well relative to this weighted criteria.
However, given that the performance of an algorithm can be arbitrarily different across instances,
no such optimal algorithm can exist under this criterion.
The question we address here is whether other measures of instance difficulty might be used to distinguish some algorithms as naturally advantageous over others.

In a recent study, \citet{jin2020pessimism} show that the pessimistic algorithm is minimax optimal when weighting each instance by the variance induced by the optimal policy.
In another recent paper,
\citet{BuGeBe20} point out that the pessimistic choice has the property that its regret improves whenever the optimal choice's value is easier to predict.
In particular, with our notation, their most relevant result (Theorem 3) implies the following:
if $b_i$ defines an interval such that $\mu_i \in [\hat \mu_i-b_i,\hat \mu_i+b_i]$ for all $i\in [k]$,
then for $i' = \argmax_i \hat \mu_i - b_i$ one obtains
\footnote{
This inequality follows directly from the definitions:
$\mu^* - \mu_{i'} \le \mu^* - (\hat \mu_{i'} - b_{i'}) \le \mu^* - (\hat \mu_{a^*} - b_{a^*}) \le 2 b_{a^*}$
and we believe this was known as a folklore result, although we 
are not able to point to a previous paper that includes
this inequality. The logic of this inequality is the same as that used in proving regret bounds for UCB policies \citep{lai1985asymptotically,lattimore2020bandit}.
It is also clear that the result holds for any data-driven stochastic optimization problem regardless of the structure of the problem. Theorem 3 of \citet{BuGeBe20} with this notation states that $\mu^*-\mu_{i'}\le \min_i \mu^*-\mu_i + 2 b_i$.
}
\begin{align*}
\mu^* - \mu_{i'} \le 2 b_{a^*}\, .
\addeq\label{eq:buckman-result}
\end{align*}
If we (liberally) interpret $b_{a^*}$ as a measure of how hard it is to predict the value of the optimal choice,
this inequality suggests that the pessimistic choice could be justified as
the choice that makes the regret comparable to the error of
predicting the optimal value.

To make this intuition precise, consider the same problem setup as discussed in Section \ref{sec:setup}.
Suppose that the reward distribution for each arm $i\in[k]$ is a Gaussian with unit variance.
Consider the problem of estimating the optimal value $\mu^*$ where the optimal arm $a^*$ is also provided to the estimator.
We define the set of minimax optimal estimators.

\vspace{-1mm}
\begin{definition}[Minimax Estimator]
For fixed $\bn = (n_i)_{i\in [k]}$, an estimator is said to be minimax optimal if its worst case error is bounded by the minimax estimate error of the problem up to some constant.
We define the set of minimax optimal estimators as
\begin{align*}
\cV^*_\bn \!=\! \left\{ \nu:\! \sup_{\theta\in\Theta_\bn}\! \EE_\theta[|\mu^*-  \nu|] \leq c \inf_{\nu'\in\cV} \sup_{\theta\in\Theta_\bn}\! \EE_{\theta}[|\mu^* - \nu'|] \right\}
\end{align*}
where $c$ is a universal constant,
and $\cV$ is the set of all possible estimators.
\end{definition}
\vspace{-1mm}

Now consider using
this optimal value estimation problem as a measure of how difficult a problem instance is, 
and then use this
to weight each problem instance as in the definition of instance-dependent lower bound.
In particular, let
\begin{align*}
\cE^*(\theta) = \inf_{\nu\in\cV^*_\bn}\EE_{\theta}[|\mu^* - \nu|]
\end{align*}
be the inherent difficulty of estimating the optimal value $\mu^*$ on problem instance $\theta$.
The previous result (\ref{eq:buckman-result}) suggests (but does not prove) that $\sup_{\theta}\frac{\cR(\lcbalg,\theta)}{\cE^*(\theta)}<+\infty$.
We now show that not only does this hold, but up to a constant factor, the LCB algorithm is nearly weighted minimax optimal with the weighting given by $\cE^*(\theta)$.

\vspace{-1mm}
\begin{proposition}
For any $\bn = (n_i)_{i\in [k]}$, 
\begin{align*}
{\sup_{\theta\in\Theta_\bn} \frac{\cR(\lcbalg,\theta)}{\cE^*(\theta)}}<c\sqrt{\log |\bn|} \, ,
\end{align*}
where $c$ is some universal constant.
\label{prop:weighted-minimax-lcb}
\end{proposition}
\vspace{-1mm}

\begin{proposition}
There exists a sequence $\{\bn_j\}$ such that
\begin{align*}
&\limsup_{j\rightarrow\infty} \sup_{\theta\in\Theta_{\bn_j}}{\frac{\cR(\textnormal{UCB},\theta)}{\sqrt{\log |\bn_j|}\cdot \cE^*(\theta)}} = + \infty  \\
&\limsup_{j\rightarrow\infty}\sup_{\theta\in\Theta_{\bn_j}} {\frac{\cR(\textnormal{greedy},\theta)}{\sqrt{\log |\bn_j|}\cdot \cE^*(\theta)}} = + \infty
\end{align*}
\label{prop:weighted-minimax-ucb-greedy}
\end{proposition}
That is, the pessimistic algorithm can be justified by weighting each instance using the difficulty of predicting the optimal value.
We note that this result does not contradict the no-instance-optimality property of batch policy optimization with stochastic bandits (Corollary~\ref{cor:instance-negative}). In fact, it only provides a characterization of pessimism: the pessimistic choice is beneficial when the batch dataset contains enough information that is good for predicting the optimal value.

\section{Related work}
\label{sec:related-work}

In the context of offline bandit and RL, a number of approaches based on the pessimistic principle have been proposed  
and demonstrate great success in practical problems \citep{swaminathan2015batch,wu2019behavior,jaques2019way,kumar2019stabilizing,kumar2020conservative,BuGeBe20,KiRaNeJo20,yu2020mopo,siegel2020keep}. 
We refer interested readers to the survey by \citet{levine2020offline} for recent developments on this topic.
To implement the pessimistic principle, the distributional robust optimization~(DRO) becomes one powerful tool in bandit~\citep{FaTaVaSmDo19,karampatziakis2019empirical} and RL~\citep{xu2010distributionally,yu2015distributionally,yang2017convex,chen2019distributionally,dai2020coindice,DeSh20}. 


In terms of theoretical perspective, the statistical properties of general DRO, \eg, the consistency and asymptotic expansion of DRO, is analyzed in~\citep{DuGlNam16}. 
\citet{liu2020provably} provides regret analysis for a pessimistic algorithm based on stationary distribution estimation in offline RL with insufficient data coverage. 
\citet{BuGeBe20} justify the pessimistic algorithm by providing an upper bound on worst-case suboptimality. 
\citet{jin2020pessimism}, \citet{KiRaNeJo20} and \citet{yin2021near} recently prove that the pessimistic algorithm is nearly minimax optimal for batch policy optimization. However, the theoretical justification of the benefits of pessimitic principle vs. alternatives are missing in offline RL. 


Decision theory motivates DRO with an axiomatic characterization of min-max (or distributionally robust) utility: Preferences of decision makers who face an uncertain decision problem and whose preference relationships over their choices satisfy certain axioms follow an ordering given by assigning max-min utility to these preferences~\citep{GiSch89}. Thus, if we believe that the preferences of the user follow the axioms stated in the above work, one must use a distributionally optimal  (pessimistic) choice. On the other hand, \citet{smith2006optimizer} raise the ``optimizer's curse'' due to statistical effect, which describes the phenomena that the resulting decision policy may disappoint on unseen out-of-sample data, \ie, the actual value of the candidate decision is below the predicted value. 
\citet{VPEKuhn17,SuVPKuhn20} justify the optimality of DRO in combating with such an overfitting issue to avoid the optimizer's curse. Moreover, \citet{DeKiWi19} demonstrate the benefits of randomized policy from DRO in the face of uncertainty comparing with deterministic policy.
While reassuring, these still leave open the question whether there is a justification for the pessimistic choice dictated by some alternate logic, or perhaps a more direct logic reasoning in terms of regret in decision problem itself~\citep{lattimore2020bandit}.
Our theoretical analysis answer this question, and provide a complete and direct justification for all confidence-based index algorithms.

\section{Conclusion}

In this paper we study the statistical limits of batch policy optimization with finite-armed bandits. 
We introduce a family of confidence-adjusted index algorithms that provides a general analysis framework to unify the commonly used optimistic and pessimistic principles. 
For this family, we show that any index algorithm with an appropriate adjustment is nearly minimax optimal. 
Our analysis also reveals another important finding, that for any algorithm that performs optimally in some environment, there exists another environment where the same algorithm can suffer arbitrarily large regret. Therefore, the instance-dependent optimality cannot be achieved by any algorithm. 
To distinguish the algorithms in offline setting, we introduce a weighted minimax objective and justify the pessimistic algorithm is nearly optimal under this criterion.

\bibliography{main}
\bibliographystyle{icml2021}

\endgroup


\newpage
\clearpage

\appendix
\onecolumn

\begin{appendix}

\thispagestyle{plain}
\begin{center}
{\huge Appendix}
\end{center}
\appendix

\section{Experiment Details}

\paragraph{Figure~\ref{fig:bandit}}  
The reward distribution for each arm $i\in[100]$ is a Gaussian with unit variance. 
The mean rewards $\mu_i$ are uniformly spread over $[0,1]$. In particular, we have $\mu_1\geq\dots\mu_{100}$, $\mu_{i}-\mu_{i+1}=0.01$ for $1\leq i < 99$, and $\mu_1=1$.  
When generating the data set, we split the arms into two sets $S_1$ and $S_2=[k] \setminus S_2$. 
For each arm $i\in\cS_1$, we collect $\pi n $ data; for each arm $i\in S_2$, we collect $n(1 - \pi |S_1|) / |S_2|$ data, 
where $n$ is the total sample size, and $0\leq \pi \leq 1 / |\cS_1|$ is a parameter to be chosen to generate different data sets. 
We consider four data sets:
\emph{LCB}-1 ($S_1=\{1\}$, $\pi=0.3$); 
\emph{LCB}-2 ($S_1=\{10\}$, $\pi=0.3$); 
\emph{UCB}-1 ($S_1=\{1\}$, $\pi=1e^{-4}$); 
\emph{UCB}-2 ($S_1=\{1,\dots,10\}$, $\pi=1e^{-4}$). 
For each instance, we run each algorithm $500$ times and use the average performance to approximate the expected
simple regret. Error bars are the standard deviation of the simple regret over the $500$ runs.
 
\paragraph{ Figure~\ref{fig:two-arm-1} and \ref{fig:two-arm-2}} 

The reward distribution for each arm $i\in[2]$ is a Gaussian with unit variance. 
We fix $\mu_1=0$ and vary $\mu_2$ accordingly. 
In Figure~\ref{fig:two-arm-1}, $n_1=10, n_2=5$. In Figure~\ref{fig:two-arm-2}, $n_1=100, n_2=10$. 
For each instance, we run each algorithm $100$ times and use the average performance to approximate the expected
simple regret. Error bars are the standard deviation of the simple regret over the $100$ runs.

\paragraph{Figure~\ref{fig:frac-1} and \ref{fig:frac-2}} 

For each $k$,  we first sample 100 vectors $\vec{\mu}=[\mu_1,...,\mu_k]$ in the following way:
We generate $\vec{\mu}_0$ with $\vec{\mu}_{0,i}=\frac{i-1}{2(k-1)} + \frac{1}{4}$ such that all reward means are 
evenly distributed with in $[\frac{1}{4}, \frac{3}{4}]$. We then add independent Gaussian noise with standard deviation $0.05$ to each $\vec{\mu}_{0,i}$ to get a sampled $\vec{\mu}$. Generating $100$ noise vectors with size $k$ gives $100$ samples of $\vec{\mu}$.
For each $\vec{\mu}$
we uniformly sample $100$~(if exist) subsets $S\subset{k}, |S|=m$~($m=k/2$ in (c) and $m=k/4$ in (d)), to generate up to $10$k instances. We set $n_i=100$ for $i\in S$ and $n_i=1$ for $i\notin S$.
For each instance, we run each algorithm $100$ times and use the average performance to approximate the expected
simple regret. We then select the algorithm with the best average performance for each instance and
count the fraction of instances where each algorithm performs the best.
Experiment details are provided in the supplementary material. 
Error bars are representing the standard deviation of the reported fraction over $5$ different runs of the whole procedure.

\section{Proof of Minimax Results}

\subsection{Proof of \cref{thm:minmax-lb}}

\newcommand{\BR}{\mathcal B \mathcal R}


Let $m \geq 2$ and $\mu^1,\ldots,\mu^m$ be a collection of vectors in $\RR^k$ with
$\mu^b_a = \Delta \sI\{a = b\}$ where $\Delta > 0$ is a constant to be chosen later.
Next, let $\theta_b$ be the environment in $\Theta_{\bn}$ with $P_a$ a Gaussian distribution with mean $\mu^b_a$ and unit variance.
Let $B$ be a random variable uniformly distributed on $[m]$ where $m \in [k]$.
The Bayesian regret of an algorithm $\cA$ is
\begin{align*}
\BR^* = \inf_{\cA} \EE\left[\cR(\cA, \theta_B)\right]
= \Delta \EE\left[\sI\{A \neq B\}\right]\,,
\end{align*}
where $A \in [k]$ is the $\sigma(X)$-measurable random variable representing the decision of the Bayesian optimal policy, which 
is $A = \argmax_{b \in [k]} \sP\{B = b | X\}$. By Bayes' law and the choice of uniform prior,
\begin{align*}
\sP\{B = b | X\} 
&\propto \exp\left(-\frac{1}{2}\sum_{a=1}^k n_a(\hat \mu_a - \mu^b_a)^2\right) \\
&= \exp\left(-\frac{1}{2} \sum_{a=1}^k n_a(\hat \mu_a - \Delta \sI\{a = b\})^2\right)\,.
\end{align*}
Therefore, the Bayesian optimal policy chooses
\begin{align*}
A = \argmin_{b \in [k]} n_b (\Delta/2 - \hat \mu_b) \,.
\end{align*}
On the other hand,
\begin{align*}
\BR^*
&= \Delta \sP\{A \neq B\} = \frac{\Delta}{k} \sum_{b=1}^k \sP_b(A \neq b)\,,
\end{align*}
where $\sP_b = \sP\{\cdot | B = b\}$.
Let $b \in [m]$ be arbitrary. Then,
\begin{align*}
&\sP_b\{A \neq b\} \\
&\geq \sP_b\left\{\hat \mu_b \leq \Delta \text{ and } \max_{a \in [m] \setminus \{b\}} \hat \mu_a \geq \frac{\Delta}{2}\left(1 + \frac{n_b}{n_a}\right)\right\} \\
&\geq \frac{1}{2} \left(1 - \prod_{a \in [m] \setminus \{b\} } \left(1 - \sP_b\left\{\hat \mu_a \geq \frac{\Delta}{2} \left(1 + \frac{n_b}{n_a}\right)\right\}\right)\right) \\
&\geq \frac{1}{2} \left(1 - \prod_{a > b} \left(1 - \sP_b\left\{\hat \mu_a \geq \Delta\right\}\right)\right)\,,
\end{align*}
where in the second inequality we used independence and the fact that the law of $\hat \mu_b$ under $\sP_b$ is Gaussian with mean $\Delta$ and variance $1/n_b$.
The first inequality follows because
\begin{align*}
\left\{\hat \mu_b \leq \Delta \text{ and } \max_{a \neq b} \hat \mu_a \geq \frac{\Delta}{2}\left(1 + \frac{n_b}{n_a}\right) \right\} \subset \{A \neq b\} \,.
\end{align*}
Let $b < a \leq m$ and
\begin{align*}
\delta_a(\Delta) = \frac{1}{\Delta \sqrt{n_a} + \sqrt{4 + n_a \Delta^2}} \sqrt{\frac{2}{\pi}} \exp\left(-\frac{n_a \Delta^2}{2}\right) \,.
\end{align*}
Since for $a \neq b$, $\hat \mu_a$ has law $\cN(0, 1/n_a)$ under $\sP_b$, by standard Gaussian tail inequalities \citep[\S26]{AS88},
\begin{align*}
\sP_b\{\hat \mu_a \geq \Delta\}
= \sP_b\{\hat \mu_a \sqrt{n_a} \geq \Delta \sqrt{n_a}\} \geq \delta_a(\Delta) \geq \delta_m(\Delta)\,,
\end{align*}
where the last inequality follows from our assumption that $n_1 \leq \cdots \leq n_k$.
Therefore, choosing $\Delta$ so that $\delta_m(\Delta) = 1/(2m)$,
\begin{align*}
\BR^*
&\geq \frac{\Delta}{2m} \sum_{b \in [m]} \left(1 - (1 - \delta_m(\Delta))^{m-b}\right) \\
&\geq \frac{\Delta}{2m} \sum_{b \in [m]} \left(1 - \left(1 - \frac{1}{2m}\right)^{m-b}\right) \\
&\geq \frac{\Delta}{2m} \sum_{b \leq m/2} \left(1 - \left(1 - \frac{1}{2m}\right)^{m/2}\right) \\
&\geq \frac{\Delta (m-1)}{20m} \geq \frac{\Delta}{40}\,.
\end{align*}
A calculation shows there exists a universal constant $c > 0$ such that
\begin{align*}
\Delta \geq c \sqrt{\frac{\log(m)}{n_m}}\,,
\end{align*}
which shows there exists a (different) universal constant $c > 0$ such that
\begin{align*}
\inf_{\cA} \sup_{\theta} \cR(\cA, \theta) \geq \BR^* \geq \max_{m \geq 2} c\sqrt{\frac{\log(m)}{n_m}}\,.
\end{align*}
The argument above relies on the assumption that $m \geq 2$. A minor modification is needed to handle the case where $n_1$ is much smaller than $n_2$.
Let $B$ be uniformly distributed on $\{1,2\}$ and let $\theta_1, \theta_2 \in \Theta_{\bn}$ be defined as above, 
but with $\mu^1 = (\Delta, 0)$ and $\mu^2 = (-\Delta, 0)$ for some constant $\Delta > 0$ to be tuned momentarily.
As before, the Bayesian optimal policy has a simple closed form solution, which is 
\begin{align*}
A = \begin{cases}
1 & \text{if } \hat \mu_1 \geq 0 \\
2 & \text{otherwise}\,.
\end{cases}
\end{align*}
The Bayesian regret of this policy satisfies
\begin{align*}
\BR^* 
&= \frac{1}{2} \cR(\cA, \theta_1) + \frac{1}{2} \cR(\cA, \theta_2) 
\geq \frac{1}{2} \cR(\cA, \theta_1) \\
&\geq \frac{1}{2} \sP_1\{A = 2\} 
\geq \frac{\Delta}{2} \sP_1\{\hat \mu_1 < 0\} \\ 
&\geq \sqrt{\frac{2}{\pi}} \frac{\Delta}{2\Delta \sqrt{n_1} + 2\sqrt{4 + n_1 \Delta^2}} \exp\left(-\frac{n_1 \Delta^2}{2}\right) \\
&\geq \frac{1}{13} \sqrt{\frac{1}{n_1}}\,,
\end{align*}
where the final inequality follows by tuning $\Delta$.

\subsection{Proof of \cref{thm:minimax-upper}}

\begin{proof}
Let $\tilde{\mu}_i$ be the index and $i' = \argmax_i \tilde{\mu}_i$.  Then, given that (\ref{eq:confidence-interval}) is true for all arms, which is with probability at least $1-\delta$,  we have
\begin{align*}
\mu^* - \mu_{i'} & = \mu^* - \tilde{\mu}_{a^*} + \tilde{\mu}_{a^*} - \tilde{\mu}_{i'} + \tilde{\mu}_{i'} - \mu_{i'} \\
& \leq \mu^* - \tilde{\mu}_{a^*} + \tilde{\mu}_{i'} - \mu_{i'}\\
& \leq \mu^* - \hat{\mu}_{a^*} + \hat{\mu}_{i'} - \mu_{i'} + 2\sqrt{\frac{2 \log (k / \delta)}{ \min_i n_i}} \\
& \leq  \sqrt{\frac{32 \log (k / \delta)}{\min_i n_i}} \,,
\end{align*}
where the first two inequalities follow from the definition of the index algorithm, and the last follows from (\ref{eq:confidence-interval}).  
Using the tower rule gives the desired result.  
\end{proof}


\section{Proof of Instance-dependent Results}

\subsection{ Instance-dependent Upper Bound}

\begin{proof}[Proof of Theorem~\ref{thm:instance-upper-general}]

Assuming $\mu_1\ge \mu_2 \ge ...\geq\mu_k$, if we have
$\Prb{\cA(X)\ge i} \le b_i$, 
then we can write
\begin{align*}
\cR(\cA) & = \sum_{2\le i \le k} \Delta_i \Prb{\cA(X)=i} \\
& = \sum_{2\le i \le k} \Delta_i \left( \Prb{\cA(X)\ge i} - \Prb{\cA(X) \ge i + 1}\right) \\
& = \sum_{2 \le i \le k} \left(\Delta_i - \Delta_{i-1}\right) \Prb{\cA(X)\ge i} \\
& \le \sum_{2 \le i \le k} \left(\Delta_i - \Delta_{i-1}\right) b_i \\
& = \sum_{2\le i \le k} \Delta_i (b_i - b_{i+1}) \,.
\end{align*}

To upper bound $\Prb{\cA(X)\ge i}$, let $\idx_i$ be the index used by algorithm $\cA$, i.e., $\cA(X) = \argmax_i \idx_i$. Then
\begin{align*}
\Prb{\cA(X)\ge i} \le \Prb{\max_{j\ge i} \idx_j \ge \max_{j<i} \idx_j} \,.
\end{align*}


Hence we can further write 
\begin{align*}
\Prb{\cA(X)\ge i} 
& \le \Prb{\max_{j\ge i} \idx_j \ge \max_{j<i} \idx_i, \max_{j<i} \idx_j \ge \eta} \\
& + \Prb{\max_{j\ge i} \idx_j \ge \max_{j<i} \idx_i, \max_{j<i} \idx_j < \eta} \\
& \le  \Prb{\max_{j\ge i} \idx_j \ge \eta} + \Prb{\max_{j<i} \idx_j < \eta}
\,.  \addeq\label{eq:ub-eta}
\end{align*}
Next we optimize the choice of $\eta$ according to the specific choice of the index.
For this let $\idx_i = \hat{\mu}_i + b_i$.

Continuing with~\eqref{eq:ub-eta}, for the first term, by the union bound we have 
\begin{align*}
\Prb{\max_{j\ge i} \idx_j \ge \eta}  \le \sum_{j\ge i} \Prb{\idx_j \ge \eta}  \,.
\end{align*}
For each $j\ge i$, by Hoeffding's inequality we have 
\begin{align*}
\Prb{\idx_j \ge \eta} \le e^{-\frac{n_j}{2} \left(\eta - \mu_j - b_j\right)_+ ^2} \,.   
\end{align*}
For the second term in~\eqref{eq:ub-eta}, we have $\Prb{\max_{j<i} \idx_j < \eta} \le \Prb{\idx_j < \eta}$ 
for each $j < i$.

By Hoeffding's inequality we have 
\begin{align*}
\Prb{\idx_j < \eta} \le  e^{-\frac{n_j}{2} \left(\mu_j + b_j - \eta \right)_+ ^2} \, , 
\end{align*}
and thus
\begin{align*}
 \Prb{\max_{j<i} \idx_j < \eta} \le \min_{j<i} e^{-\frac{n_j}{2} \left(\mu_j + b_j - \eta \right)_+ ^2} \,.
\end{align*}
Define
\begin{align*}
g_i(\eta) =  \sum_{j\ge i} e^{-\frac{n_j}{2} \left(\eta - \mu_j - b_j\right)_+ ^2}
+ \min_{j < i} e^{-\frac{n_j}{2} \left(\mu_j + b_j - \eta \right)_+ ^2}
\end{align*}
and $g_i^* = \min_{\eta} g_i(\eta)$. Then we have
\begin{align*}
\Prb{\cA(X)\ge i} \le \min\{ 1, g_i^* \} \,. 
\end{align*}
Putting everything together, we bound the expected regret as
\begin{align*}
\cR(\cA) \le \sum_{2\le i \le k} \Delta_i \left(\min\{ 1, g_i^* \} - \min\{ 1, g_{i+1}^* \} \right)
\end{align*}
where we define $g_{k+1}^*=0$.

\end{proof}

\begin{proof}[Proof of Remark~\ref{remark:recover-minimax}]
Recall the definition of $g_i(\eta)$:
\begin{align*}
g_i(\eta) =  \sum_{j\ge i} e^{-\frac{n_j}{2} \left(\eta - \mu_j - b_j\right)_+ ^2}
+ \min_{j < i} e^{-\frac{n_j}{2} \left(\mu_j + b_j - \eta \right)_+ ^2}\, .
\end{align*}
Let $\eta = \mu_1 - 2 \sqrt{\frac{2}{n_{\min}}\log\frac{k}{\delta}}$. 
Then, for the second term of $g_i(\eta)$,
\begin{align*}
    \min_{j < i} e^{-\frac{n_j}{2} \left(\mu_j + b_j - \eta \right)_+ ^2} \leq e^{-\frac{n_1}{2} \left(  2 \sqrt{\frac{2}{n_{\min}}\log\frac{k}{\delta}} -\sqrt{\frac{2}{n_{1}}\log\frac{k}{\delta}}\right)_{+}^2 } \leq \frac{\delta}{k}\, .
\end{align*}
For the first term,
\begin{align*}
    \sum_{j\ge i} e^{-\frac{n_j}{2} \left(\eta - \mu_j - b_j\right)_+ ^2} = \sum_{j\ge i} e^{-\frac{n_j}{2} \left(\mu_1 - 2 \sqrt{\frac{2}{n_{\min}}\log\frac{k}{\delta}} - \mu_j - b_j\right)_+ ^2} \leq \sum_{j\ge i} e^{-\frac{n_{\min}}{2} \left(\Delta_j - 3 \sqrt{\frac{2}{n_{\min}}\log\frac{k}{\delta}} \right)_+ ^2}\, .
\end{align*}
Thus,
\begin{align*}
g_i^* \le  \sum_{j\ge i} e^{-\frac{n_{\min}}{2} \left(\Delta_j - 3 \sqrt{\frac{2}{n_{\min}}\log\frac{k}{\delta}} \right)_+ ^2} + \frac{\delta}{k}\, .
\end{align*}
For arm $i$ such that $\Delta_i \ge 4\sqrt{\frac{2}{n_{\min}}\log\frac{k}{\delta}}$, 
by Theorem~\ref{thm:instance-upper-general} we have $P(\cA(X)\geq i) \leq g^*_i \leq \delta$. The result then follows by the tower rule.   
\end{proof}

\begin{proof}[Proof of Corollary~\ref{coro:instance-upper-general-simplified}]
For each $i$, let $\eta_i = \max_{j<i} L_j$. Then,
\begin{align*}
g_i(\eta_i) =  \sum_{j\ge i} e^{-\frac{n_j}{2} \left( \max_{j<i} L_j - \mu_j - b_j\right)_+ ^2}
+ \min_{j < i} e^{-\frac{n_j}{2} \left(\mu_j + b_j -  \max_{j<i} L_j \right)_+ ^2}\, .
\end{align*}
Let $s=\argmax_{j<i} L_j$. For the second term we have,
\begin{align*}
    \min_{j < i} e^{-\frac{n_j}{2} \left(\mu_j + b_j -  \max_{j<i} L_j \right)_+ ^2} \leq e^{-\frac{n_s}{2} \left(\mu_s + b_s -  L_s \right)_+ ^2} \leq \frac{\delta}{k}\, .
\end{align*}
Next we consider the first term. 
Recall that $h = \max\{i\in[k]: \max_{j<i} L_j < \max_{j'\ge i} U_{j'}\}$. 
Then for any $i>h$, we have $\max_{j<i} L_j \ge U_{j'}$ for all $j'\ge i$. Therefore,
\begin{align*}
    \sum_{j\ge i} e^{-\frac{n_j}{2} \left( \max_{j'<i} L_{j'} - \mu_j - b_j\right)_+ ^2}  = \sum_{j\ge i} e^{-\frac{n_j}{2} \left( \max_{j'<i} L_{j'} - U_j + \sqrt{\frac{2}{n_j} \log \frac{k}{\delta}}\right)_+ ^2} \leq \frac{\delta}{k}\sum_{j\ge i} e^{-\frac{n_j}{2} \left( \max_{j'<i} L_{j'} - U_j \right)^2}\, .
\end{align*}
Note that for $i\leq h$, $\Delta_i \leq \Delta_h$. Thus we have,
\begin{align*}
    \cR(\cA) & \leq \Delta_h + \sum_{i>h} (\Delta_i - \Delta_{i-1} )\sP(\cA(X)\geq i) \\
    & \leq \Delta_h +  \frac{\delta}{k}\Delta_{\max} + \frac{\delta}{k}\sum_{i>h} (\Delta_i - \Delta_{i-1} )\sum_{j\ge i} e^{-\frac{n_j}{2} \left( \max_{j'<i} L_{j'} - U_j \right)^2}\, ,
\end{align*}
which concludes the proof. 
\end{proof}

\begin{proof}[Proof of Corollary~\ref{coro:ub-greedy-instance}]
Considering the greedy algorithm,
for each $i\ge 2$,
\begin{align*}
g_i(\eta) =  \sum_{j\ge i} e^{-\frac{n_j}{2} \left(\eta - \mu_j\right)_+ ^2}
+ \min_{j < i} e^{-\frac{n_j}{2} \left(\mu_j - \eta \right)_+ ^2} \,.
\end{align*}
Define $h_i = \argmax_{j<i} \mu_j - \sqrt{\frac{2}{n_j}\log\frac{k}{\delta}}$ and $\eta_i =\mu_{h_i} - \sqrt{\frac{2}{n_{h_i}}\log\frac{k}{\delta}}$. Then we have
$
e^{-\frac{n_{h_i}}{2} \left(\mu_{h_i} - \eta_i \right)_+ ^2} = \delta/k
$.
Then for $j \ge i$ we have 
\begin{align*}
e^{-\frac{n_j}{2} \left(\eta_i - \mu_j\right)_+ ^2} = e^{-\frac{n_j}{2} \left(\mu_{h_i} - \mu_j - \sqrt{\frac{2}{n_{h_i}}\log\frac{k}{\delta}} \right)_+ ^2} \,.
\end{align*}

When $\mu_{h_i} - \mu_j \ge \sqrt{\frac{2}{n_{h_i}}\log\frac{k}{\delta}} + \sqrt{\frac{2}{n_{j}}\log\frac{k}{\delta}}$ we have $e^{-\frac{n_j}{2} \left(\eta_i - \mu_j\right)_+ ^2} \le \delta/k$.

Define 
\begin{align*}
U_i = \I{\forall j \ge i, \mu_{h_i} - \mu_j \ge \sqrt{\frac{2}{n_{h_i}}\log\frac{k}{\delta}} + \sqrt{\frac{2}{n_{j}}\log\frac{k}{\delta}}} \,.
\end{align*}
Then we have
$
g_i^* U_i \le \frac{k - i + 2}{k}\delta \le \delta
$.
According to Theorem~\ref{thm:instance-upper-general} we have $\Prb{\cA(X)\ge i} \le \min\{ 1, g_i^* \}$,
so for any $i$ such that $\Prb{\cA(X)\ge i} > \delta$, we must have $U_i = 0$, which is equivalent to 
\begin{align*}
\max_{j<i} \mu_j - \sqrt{\frac{2}{n_{j}}\log\frac{k}{\delta}} < \max_{j \ge i} \mu_j  + \sqrt{\frac{2}{n_{j}}\log\frac{k}{\delta}} \, .
\addeq\label{eq:proof-greedy-cond-1}
\end{align*}
Let $\hat{i}$ be the largest index $i$ that satisfies \eqref{eq:proof-greedy-cond-1}. 
Then we have $\Prb{\cA(X)\ge \hat{i} + 1} \le \delta$. Therefore, we have $\Prb{\mu^* - \mu_{\cA(X)} \le \Delta_{\hat{i}}} \ge 1 - \delta$, and it remains to upper bound $\Delta_{\hat{i}}$.

For any $i \in \iset{k}$, if $\hat{i} \le i$ then $\Delta_{\hat{i}} \le \Delta_{i}$. If $\hat{i} > i$ we have 
\begin{align*}
\max_{j<\hat{i}} \mu_j - \sqrt{\frac{2}{n_{j}}\log\frac{k}{\delta}} \ge \mu_{i} - \sqrt{\frac{2}{n_{i}}\log\frac{k}{\delta}}
\end{align*}
and
\begin{align*}
\max_{j \ge \hat{i}} \mu_j  + \sqrt{\frac{2}{n_{j}}\log\frac{k}{\delta}} \le \mu_{\hat{i}} + \max_{j > i} \sqrt{\frac{2}{n_{j}}\log\frac{k}{\delta}} \,.
\end{align*}
Applying \eqref{eq:proof-greedy-cond-1} gives
\begin{align*}
\Delta_{\hat{i}} - \Delta_{i} = \mu_{i} - \mu_{\hat{i}} \le \sqrt{\frac{2}{n_{i}}\log\frac{k}{\delta}} +  \max_{j > i} \sqrt{\frac{2}{n_{j}}\log\frac{k}{\delta}} \,,
\end{align*}
so
\begin{align*}
\Delta_{\hat{i}} \le \Delta_{i} + \sqrt{\frac{2}{n_{i}}\log\frac{k}{\delta}} +  \max_{j > i} \sqrt{\frac{2}{n_{j}}\log\frac{k}{\delta}} \,
\end{align*}
holds for any $i \in \iset{k}$, concluding the proof.
\end{proof}

\begin{proof}[Proof of Corollary~\ref{coro:ub-lcb-instance}]
Let $\eta = \max_i \mu_i -  \sqrt{\frac{8}{n_i}\log\frac{k}{\delta}}$. 
Considering the LCB algorithm,
for each $i\ge 2$, we have
\begin{align*}
g_i(\eta) =  & \sum_{j\ge i} e^{-\frac{n_j}{2} \left(\eta - \mu_j + \sqrt{\frac{2}{n_j}\log\frac{k}{\delta}} \right)_+ ^2} + \min_{j < i} e^{-\frac{n_j}{2} \left(\mu_j - \eta - \sqrt{\frac{2}{n_j}\log\frac{k}{\delta}} \right)_+ ^2} \,.
\end{align*}

Define $h_i = \argmax_{j<i} \mu_j - \sqrt{\frac{8}{n_j}\log\frac{k}{\delta}}$ and $\eta_i =\mu_{h_i} - \sqrt{\frac{8}{n_{h_i}}\log\frac{k}{\delta}}$. Then we have
$
e^{-\frac{n_{h_i}}{2} \left(\mu_{h_i} - \eta_i - \sqrt{\frac{2}{n_j}\log\frac{k}{\delta}} \right)_+ ^2} = \delta/k
$.
Now, consider $j \ge i$. Then,
\begin{align*}
e^{-\frac{n_j}{2} \left(\eta_i - \mu_j + \sqrt{\frac{2}{n_j}\log\frac{k}{\delta}}\right)_+ ^2} \le \frac{\delta}{k}
\end{align*}
whenever $\eta_i - \mu_j \ge 0$, i.e. $\mu_{h_i} - \sqrt{\frac{8}{n_{h_i}}\log\frac{k}{\delta}} \ge \mu_j$.

Define
\begin{align*}
U_i = \I{\forall j \ge i, \mu_{h_i} - \sqrt{\frac{8}{n_{h_i}}\log\frac{k}{\delta}} \ge \mu_j } \,.
\end{align*}
Then we have
$
g_i^* U_i \le \frac{k - i + 2}{k}\delta \le \delta
$.
According to Theorem~\ref{thm:instance-upper-general} we have $\Prb{\cA(X)\ge i} \le \min\{ 1, g_i^* \}$,
so for any $i$ such that $\Prb{\cA(X)\ge i} > \delta$, we must have $U_i = 0$, which is equivalent to that there exists some $s\geq i$ such that
\begin{align*}
\mu_s > \max_{j<i} \mu_j - \sqrt{\frac{8}{n_j}\log\frac{k}{\delta}}\,.
\addeq\label{eq:proof-lcb-cond-1}
\end{align*}
Let $\hat{i}$ be the largest index $i$ that satisfies \eqref{eq:proof-lcb-cond-1}. 
Then we have $\Prb{\cA(X)\ge \hat{i} + 1} \le \delta$ and thus $\Prb{\mu^* - \mu_{\cA(X)} \le \Delta_{\hat{i}}} \ge 1 - \delta$. It remains to upper bound $\Delta_{\hat{i}}$.

For any $i \in \iset{k}$, if $\hat{i} \le i$ then $\Delta_{\hat{i}} \le \Delta_{i}$. If $\hat{i} > i$ we have 
\begin{align*}
\mu_{\hat{i}} > \max_{j < \hat{i}} \mu_j - \sqrt{\frac{8}{n_j}\log\frac{k}{\delta}}  \ge \mu_i - \sqrt{\frac{8}{n_i}\log\frac{k}{\delta}} \,.
\end{align*}
Therefore,
\begin{align*}
\Delta_{\hat{i}} = \Delta_i + \mu_{i} - \mu_{\hat{i}} \le \Delta_i + \sqrt{\frac{8}{n_{i}}\log\frac{k}{\delta}} \,,
\end{align*}
 which concludes the proof.
\end{proof}

\begin{proof}[Proof of Corollary~\ref{coro:ub-ucb-instance}]
Consider now the UCB algorithm. Then, 
for each $i\ge 2$,
\begin{align*}
g_i(\eta) =  & \sum_{j\ge i} e^{-\frac{n_j}{2} \left(\eta - \mu_j - \sqrt{\frac{2}{n_j}\log\frac{k}{\delta}} \right)_+ ^2} + \min_{j < i} e^{-\frac{n_j}{2} \left(\mu_j - \eta + \sqrt{\frac{2}{n_j}\log\frac{k}{\delta}} \right)_+ ^2} \,.
\end{align*}
Pick $\eta=\mu_1$ then the second term in $g_i(\eta)$ becomes $\delta/k$. For $j$ such that $\Delta_j\ge \sqrt{\frac{8}{n_j}\log\frac{k}{\delta}}$ we have 
\begin{align*}
e^{-\frac{n_j}{2} \left(\eta - \mu_j - \sqrt{\frac{2}{n_j}\log\frac{k}{\delta}} \right)_+ ^2} \le \frac{\delta}{k} \,.
\end{align*}
Define
\begin{align*}
U_i = \I{\forall j \ge i, \Delta_j\ge \sqrt{\frac{8}{n_j}\log\frac{k}{\delta}} } \,.
\end{align*}
Then we have
$
g_i^* U_i \le \frac{k - i + 2}{k}\delta \le \delta
$.
According to Theorem~\ref{thm:instance-upper-general} we have $\Prb{\cA(X)\ge i} \le \min\{ 1, g_i^* \}$,
so for any $i$ such that $\Prb{\cA(X)\ge i} > \delta$, we must have $U_i = 0$, which is equivalent to
\begin{align*}
\max_{j \ge i} \mu_j + \sqrt{\frac{8}{n_j}\log\frac{k}{\delta}} > \mu_1\,.
\addeq\label{eq:proof-ucb-cond-1}
\end{align*}
Let $\hat{i}$ be the largest index $i$ that satisfies \eqref{eq:proof-ucb-cond-1}. 
Then we have $\Prb{\cA(X)\ge \hat{i} + 1} \le \delta$. Therefore, we have $\Prb{\mu^* - \mu_{\cA(X)} \le \Delta_{\hat{i}}} \ge 1 - \delta$. It remains to upper bound $\Delta_{\hat{i}}$.

For any $i \in \iset{k}$, if $\hat{i} \le i$ then $\Delta_{\hat{i}} \le \Delta_{i}$. If $\hat{i} > i$, we have 
\begin{align*}
\max_{j \ge \hat{i}} \mu_j + \sqrt{\frac{8}{n_j}\log\frac{k}{\delta}}  \le \mu_{\hat{i}} + \max_{j > i} \sqrt{\frac{8}{n_{j}}\log\frac{k}{\delta}} \,.
\end{align*}
Applying \eqref{eq:proof-ucb-cond-1} gives
\begin{align*}
\Delta_{\hat{i}} = \mu_{1} - \mu_{\hat{i}} \le \max_{j > i} \sqrt{\frac{8}{n_{j}}\log\frac{k}{\delta}} \,.
\end{align*}
Therefore,
\begin{align*}
\Delta_{\hat{i}} 
\le \max\left\{ \Delta_{i}, \max_{j > i} \sqrt{\frac{8}{n_{j}}\log\frac{k}{\delta}} \right\} \le \Delta_{i} + \max_{j > i} \sqrt{\frac{8}{n_{j}}\log\frac{k}{\delta}} \,.
\end{align*}
for any $i \in \iset{k}$, which concludes the proof.
\end{proof}

\begin{proof}[Proof of Proposition~\ref{prop:lcb-vs-ucb}]

Fixing $S\subset \iset{k}$, we take $\{n_i\}_{i\in S}\to\infty$ and $\{n_i\}_{i\notin S} = 1$. The upper bound for LCB in Corollary~\ref{coro:ub-lcb-instance} can be written as

\begin{align*}
\hat{\cR}_S(\lcbalg)
& = \min\left\{ \min_{i\in S} \Delta_i , \min_{i\notin S} \left(\Delta_i + \sqrt{8\log\frac{k}{\delta}}\right) \right \}+ \delta \\
& = \min_{i\in S} \Delta_i + \delta \\
& = \Delta_{\min\{i\in \iset{k}: i\in S\}} + \delta \,.
\end{align*}
Similarly, we have
\begin{align*}
\hat{\cR}_S(\ucbalg)
 =  \min_{i\in \iset{k}} \left(\Delta_i + \max_{j > i, j\notin S} \sqrt{8\log\frac{k}{\delta}} \right) + \delta
\end{align*}
and
\begin{align*}
\hat{\cR}_S(\greedyalg)
\ge \min_{i\in \iset{k}} \left(\Delta_i + \max_{j > i, j\notin S} \sqrt{2\log\frac{k}{\delta}} \right) + \delta \,.
\end{align*}
Note that for $\delta\in (0,1)$, $\sqrt{2\log\frac{k}{\delta}} > 1 \ge \Delta_{\max}$. So we can further lower bound 
$\hat{\cR}_S(\ucbalg)$ and $\hat{\cR}_S(\greedyalg)$ by $\Delta_h + \delta$ where
$h=\min\{i \in \iset{k}: \forall j > i, j\in S \}$. Let $m=|S|$. Notice that unless $S=\{k - m + 1, ..., k\}$, we always have
$\min\{i\in \iset{k}: i\in S\} < \min\{i \in \iset{k}: \forall j > i, j\in S \}$. So we have $\hat{\cR}_S(\lcbalg) <
\hat{\cR}_S(\ucbalg)$~(or $\hat{\cR}_S(\greedyalg)$) whenever $S\ne \{k - m + 1, ..., k\}$. 
Under the uniform distribution over all possible subsets for $S$, the event $S=\{k - m + 1, ..., k\}$ 
happens with probability $\binom{k}{m}^{-1}$, which concludes the proof.

\end{proof}

\subsection{Instance-dependent Lower Bounds}

\begin{proof}[Proof of Theorem \ref{thm:instance-negative-1}]


We first derive an upper bound for $\cR^*_{\cM_{\bn, c}}(\theta)$.
Assuming $X=(X_1, X_2, \bn)$ with $X_i\sim \cN(\mu_i, 1/n_i)$, for any $\beta\in \mathbb{R}$, we define algorithm 
$\cA_{\beta}$ as 
\begin{align*}
\cA_{\beta}(X) = 
\begin{cases}
1, & \textrm{ if } X_1 - X_2 \ge \frac{\beta}{\sqrt{n_{\min}}}\, ; \\
2, & \textrm{ otherwise } \,.
\end{cases}
\end{align*}

We now analyze the regret for $\cA_{\beta}$. By Hoeffding's inequality we have the following instance-dependent 
regret upper bound:
\begin{proposition}
Consider any $\beta\in \mathbb{R}$ and $\theta\in\Theta_{\bn}$. Let $\Delta=|\mu_1 - \mu_2|$. If $\mu_1\ge \mu_2$ then
\begin{align*}
\cR(\cA_{\beta}, \theta) 
\le \I{\Delta \le \frac{\beta}{\sqrt{n_{\min}}}} \frac{\beta}{\sqrt{n_{\min}}} + \I{\Delta > \frac{\beta}{\sqrt{n_{\min}}} }e^{-\frac{n_{\min}}{4}\left( \Delta - \frac{\beta}{\sqrt{n_{\min}}} \right)_+^2}\,.
\end{align*}
Furthermore, if $\mu_1< \mu_2$, we have
\begin{align*}
\cR(\cA_{\beta}, \theta) 
\le \I{\Delta \le \frac{-\beta}{\sqrt{n_{\min}}}} \frac{-\beta}{\sqrt{n_{\min}}} + \I{\Delta > \frac{-\beta}{\sqrt{n_{\min}}} }e^{-\frac{n_{\min}}{4}\left( \Delta + \frac{\beta}{\sqrt{n_{\min}}} \right)_+^2}\,.
\end{align*}
\end{proposition}
Maximizing over $\Delta$ gives our worst case regret guarantee: 
\begin{proposition}
For any $\beta\in \mathbb{R}$,
\begin{align*}
\sup_{\theta\in \Theta_{\bn}} \cR(\cA_{\beta}, \theta) \le \frac{|\beta| + 2}{\sqrt{n_{\min}}} \,.
\end{align*}
\end{proposition}
$\cA_{\beta}(X)$ is minimax optimal for a specific range of $\beta$:
\begin{proposition}
If $|\beta| \le cc_0 - 2$ then $\cA_{\beta} \in \cM_{\bn, c}$.
\label{prop:thresholding-alg-upper}
\end{proposition}

Given $\theta\in \Theta_{\bn}$, to upper bound $\cR^*_{\cM_{\bn, c}}(\theta)$, we pick $\beta$ such that $\cA_{\beta} \in \cM_{\bn, c}$ and $\cA_{\beta}$ performs well on $\theta$. For $\theta$ where $\mu_1\ge \mu_2$,
we set $\beta = 2 - cc_0$ thus $\cR^*_{\cM_{\bn, c}}(\theta) \le \cR(\cA_{2-cc_0}, \theta)$. For $\theta$ where $\mu_1 < \mu_2$,
we set $\beta = cc_0 - 2$ thus $\cR^*_{\cM_{\bn, c}}(\theta) \le \cR(\cA_{cc_0-2}, \theta)$.

We now construct two instances $\theta_1, \theta_2\in \Theta_{\bn}$ and show that no algorithm can achieve regret close to $\cR^*_{\cM_{\bn, c}}$ on both instances. Fixing some $\lambda \in \mathbb{R}$ and $\eta > 0$, we define
$$\theta_1=(\mu_1, \mu_2)=(\lambda + \frac{\eta}{n_1}, \lambda - \frac{\eta}{n_2})$$ and 
$$\theta_2=(\mu'_1, \mu'_2)=(\lambda - \frac{\eta}{n_1}, \lambda + \frac{\eta}{n_2})\,.$$
On instance $\theta_1$ we have $X_1-X_2\sim \cN(( \frac{1}{n_1} + \frac{1}{n_2} )\eta, \frac{1}{n_1} + \frac{1}{n_2} )$ while on instance $\theta_2$ we have  $X_1-X_2\sim \cN(-( \frac{1}{n_1} + \frac{1}{n_2} )\eta,  \frac{1}{n_1} + \frac{1}{n_2} )$. 
Let $\Phi$ be the CDF of the standard normal distribution $\cN(0, 1)$, $\Delta=( \frac{1}{n_1} + \frac{1}{n_2} )\eta$, and $\sigma^2=\frac{1}{n_1} + \frac{1}{n_2}$. Then we have
\begin{align*}
\cR(\cA_{\beta}, \theta_1) 
& = \Delta \Prbb{\theta_1}{\cA_{\beta} = 2} \\ 
& =  \Delta \Prbb{\theta_1}{X_1-X_2 < \frac{\beta}{\sqrt{n_{\min}}} } \\
& =  \Delta \Phi\left( \frac{\beta - \Delta \sqrt{n_{\min}} }{\sigma \sqrt{n_{\min}} } \right)\,,
\end{align*}
and 
\begin{align*}
\cR(\cA_{-\beta}, \theta_2) 
& = \Delta \Prbb{\theta_2}{\cA_{-\beta} = 1} \\ 
& =  \Delta \Prbb{\theta_2}{X_1-X_2 \ge - \frac{\beta}{\sqrt{n_{\min}}} } \\
& =  \Delta \Phi\left( \frac{\beta - \Delta \sqrt{n_{\min}} }{\sigma \sqrt{n_{\min}} } \right)\,.
\end{align*}
It follows that our upper bound on $\cR^*_{\cM_{\bn, c}}$ is the same for both instances, i.e., $\cR(\cA_{2-cc_0}, \theta_1) = \cR(\cA_{cc_0-2}, \theta_2)$. 
Next we show that the greedy algorithm $\cA_0$ is optimal in terms of minimizing the worse regret between $\theta_1$ and $\theta_2$.

\begin{lemma}
Let $\cA_0$ be the greedy algorithm where $\cA_0(X)=1$ if $X_1\ge X_2$ and $\cA_0(X)=2$ otherwise. Then we have
\begin{align*}
\cR(\cA_0, \theta_1)= \cR(\cA_0, \theta_2) = \min_{\cA} \max\{\cR(\cA, \theta_1), \cR(\cA, \theta_2)\} \,.
\end{align*}
\label{lemma:greedy-optimal-two-ins}
\end{lemma}

\begin{proof}[Proof of Lemma~\ref{lemma:greedy-optimal-two-ins}]

The first step is to show that by applying the Neyman-Pearson Lemma, thresholding algorithms on $X_1-X_2$ perform the most powerful 
hypothesis tests between $\theta_1$ and $\theta_2$.

Let $f_{\theta}$ be the probability density function for the observation $(X_1, X_2)$ under
instance $\theta$. Then, the likelihood ratio function can be written as
\begin{align*}
\frac{ f_{\theta_1}(X_1, X_2) }{ f_{\theta_2}(X_1, X_2) } 
= \frac{ e^{-\frac{n_1}{2} (X_1 - \lambda - \eta/n_1)^2 - \frac{n_2}{2} (X_2 - \lambda + \eta/n_2)^2 } }
{ e^{-\frac{n_1}{2} (X_1 - \lambda + \eta/n_1)^2 - \frac{n_2}{2} (X_2 - \lambda - \eta/n_2)^2 } } = e^{2\eta(X_1 - X_2)} \,.
\end{align*}

Applying the Neyman-Pearson Lemma to our scenario gives the following statement:

\begin{proposition}[Neyman-Pearson Lemma]
For any $\gamma> 0$ let $\cA^{\gamma}$ be the algorithm where $\cA^{\gamma}(X)=1$ if $\frac{ f_{\theta_1}(X_1, X_2) }{ f_{\theta_2}(X_1, X_2) } \ge \gamma$ and $\cA^{\gamma}(X)=2$ otherwise. Let $\alpha = \Prbb{\theta_1}{\cA^{\gamma}(X)=2}$. Then for any algorithm $\cA'$ such that 
$\Prbb{\theta_1}{\cA'(X)=2}=\alpha$, we have 
$
\Prbb{\theta_2}{\cA'(X)=1} \ge  \Prbb{\theta_2}{\cA^{\gamma}(X)=1}$.
\label{prop:neyman-pearson}
\end{proposition}
Note that $\frac{ f_{\theta_1}(X_1, X_2) }{ f_{\theta_2}(X_1, X_2) } \ge \gamma$ is equivalent to 
$X_1-X_2\ge (2\eta)^{-1}\log \gamma$. Returning to the proof of
Lemma~\ref{lemma:greedy-optimal-two-ins}, 
consider an arbitrary algorithm $\cA'$ and let $\alpha = \cR(\cA', \theta_1)/\Delta = \Prbb{\theta_1}{\cA'(X)=2}$.
Let $\gamma$ be the threshold that satisfies $\Prbb{\theta_1}{\cA^{\gamma}(X)=2}=\alpha$. 
This exists because $X_1, X_2$ follow a continuous distribution. 
According to
Proposition~\ref{prop:neyman-pearson} we have $
\Prbb{\theta_2}{\cA'(X)=1} \ge  \Prbb{\theta_2}{\cA^{\gamma}(X)=1}$. Therefore, we have shown that 
$\cR(\cA^{\gamma}, \theta_1)=\cR(\cA', \theta_1)$ and $\cR(\cA^{\gamma}, \theta_2)\le\cR(\cA', \theta_2)$,
which means that for any algorithm $\cA'$ there exists some $\gamma$ such that 
\begin{align*}
\max\{\cR(\cA^\gamma, \theta_1), \cR(\cA^\gamma, \theta_2)\} 
\le \max\{\cR(\cA', \theta_1), \cR(\cA', \theta_2)\} \,.  
\end{align*}
It remains to show that $\gamma=1$ is the minimizer of $\max\{\cR(\cA^\gamma, \theta_1), \cR(\cA^\gamma, \theta_2)\}$. This comes from the fact that $\cR(\cA^\gamma, \theta_1)$ is a monotonically increasing 
function of $\gamma$ while $\cR(\cA^\gamma, \theta_2)$ is a monotonically decreasing 
function of $\gamma$ and $\gamma=1$ makes $\cR(\cA^\gamma, \theta_1)=\cR(\cA^\gamma, \theta_2)$,  
which means that $\gamma=1$ is the minimizer.
\end{proof}
We now continue with the proof of Theorem~\ref{thm:instance-negative-1}. Applying Lemma~\ref{lemma:greedy-optimal-two-ins}
gives
\begin{align*}
\sup_{\theta\in \Theta_{\bn}} \frac{\cR(\cA, \theta)}{\cR^*_{\cM_{\bn, c}}(\theta)} 
& \ge  \max \left\{ \frac{\cR(\cA, \theta_1)}{\cR^*_{\cM_{\bn, c}}(\theta_1)}, \frac{\cR(\cA, \theta_2)}{\cR^*_{\cM_{\bn, c}}(\theta_2)}\right\} \\
& \ge \max \left\{ \frac{\cR(\cA, \theta_1)}{\cR(\cA_{2-cc_0}, \theta_1)}, \frac{\cR(\cA, \theta_2)}{\cR(\cA_{cc_0-2}, \theta_2)}\right\} \\
& = \frac{\max\left\{  \cR(\cA, \theta_1), \cR(\cA, \theta_2)\right\}}{\cR(\cA_{2-cc_0}, \theta_1)} \\
& \ge \frac{\cR(\cA_0, \theta_1)}{\cR(\cA_{2-cc_0}, \theta_1)} \\
& = \frac{\Phi\left( - \frac{\Delta}{\sigma } \right)}{\Phi\left( - \frac{cc_0-2}{\sigma \sqrt{n_{\min}} }
- \frac{\Delta}{\sigma } \right)} \,.
\addeq\label{eq:cdf-ratio}
\end{align*}
Now we apply the fact that for $x>0$, $\frac{x}{1 + x^2}\phi(x) < \Phi(-x) < \frac{1}{x}\phi(x)$
to lower bound \eqref{eq:cdf-ratio}, where $\phi$ is the probability density function of the standard normal distribution. Choosing $\beta = cc_0 - 2$, we have
\begin{align*}
\frac{\Phi\left( - \frac{\Delta}{\sigma } \right)}{\Phi\left( - \frac{\beta}{\sigma \sqrt{n_{\min}} }
- \frac{\Delta}{\sigma } \right)} \ge \frac{\beta + \Delta \sqrt{n_{\min}} }{\sigma \sqrt{n_{\min}} }\frac{\Delta/\sigma}{1 + (\Delta/\sigma)^2}
e^{\frac{1}{2} \left(  \frac{\beta^2}{\sigma^2 n_{\min}} + \frac{\beta\Delta }{\sigma^2 \sqrt{n_{\min}}}  \right) } \ge \frac{\eta^2}{n_{\min} + \eta^2} e^{\frac{\beta^2}{4} + \frac{\beta\eta}{2 \sqrt{n_{\min}}}  } \,. 
\end{align*}

Picking $\lambda=1/2$ and $\eta=n_{\min}/2$ such that $\theta_1, \theta_2 \in [0, 1]^2$, we have
\begin{align*}
\sup_{\theta\in \Theta_{\bn}} \frac{\cR(\cA, \theta)}{\cR^*_{\cM_{\bn, c}}(\theta)}
 \ge \frac{ n_{\min} }{ n_{\min} + 4 } e^{\frac{\beta^2}{4} +  \frac{\beta}{4}\sqrt{n_{\min}}  }\,,
\end{align*}
which concludes the proof.
\end{proof}

\section{Proof for Section~\ref{sec:pessimism}}

For any $\theta$, let $\mu_1$ and $n_1$ be the reward mean and sample count for the optimal arm. 
We first prove that $\cE^*(\theta)$ is at the order of $1/\sqrt{n_1}$ for any $\theta$.

\begin{proposition}
There exist universal constants $c_0$ and $c_1$ such that, for any $\theta \in \Theta_{\bn}$, 
$c_0/\sqrt{n_1} \le \cE^*(\theta) \le c_1/\sqrt{n_1}$.
\label{prop:mean-est}
\end{proposition}

\begin{proof}[Proof of Proposition~\ref{prop:mean-est}]

For any constant $c>0$,
define $\theta'\in \Theta$ such that the only difference between $\theta'$ and $\theta$ is the mean for the optimal arm: $\theta'$ has $\mu_1' = \mu_1 + \frac{4c}{\sqrt{n_1}}$.  

For any algorithm such that 
$\EEE{\theta'}{|\mu_1' - \nu|} \le \frac{c}{\sqrt{n_1}}$, we have $\Prbb{\theta'}{\nu \ge \mu_1 + \frac{2c}{\sqrt{n_1}}} \ge \frac{1}{2}$ by Markov inequality. Applying the fact that, when $p$ and $q$ 
are two Bernoulli distributions with parameter $p$ and $q$ respectively, if $p\ge 1/2$ we have $\mathrm{KL}(p, q)\ge \frac{1}{2}\log \frac{1}{4q}$. 
Then we have 
\begin{align*}
\Prbb{\theta}{\nu \ge \mu_1 + \frac{2c}{\sqrt{n_1}}} \ge \frac{1}{4} e^{- \mathrm{KL}(\theta, \theta')} 
= \frac{1}{4} e^{-4c^2} \,.
\end{align*}
Therefore, we have 
\begin{align*}
\EEE{\theta}{|\mu_1 - \nu|} \ge \frac{2c}{\sqrt{n_1}} \Prbb{\theta}{\nu \ge \mu_1 + \frac{2c}{\sqrt{n_1}}} \ge 
\frac{ce^{-4c^2}}{2\sqrt{n_1}} \,. 
\end{align*}
Now we apply the fact that the empirical mean estimator $\nu = \hat{\mu}_1$ has $\EEE{\theta}{|\mu_1 - \nu|}\le  \frac{1}{\sqrt{n_1}}$ for any $\theta$. We know that  $\inf_{\nu} \sup_{\theta}\EEE{\theta}{|\mu_1 - \nu|} \le \frac{1}{\sqrt{n_1}}$.  Let $c_2$ be the constant in the definition of $\cV^*_\bn$,
then $\frac{c_2e^{-4c_2^2}}{2\sqrt{n_1}}$ is a lower bound on $\cE^*(\theta)$ for any $\theta$ due to 
the fact that relaxing the constraint on the minimax optimality gives a lower instance dependent regret lower bound.
Since the minimax value is also an upper bound on $\cE^*(\theta)$ we know that, there exist universal constants $c_0$ and $c_1$ such that, for any $\theta \in \Theta_{\bn}$, 
$c_0/\sqrt{n_1} \le \cE^*(\theta) \le c_1/\sqrt{n_1}$.

\end{proof}

\begin{proof}[Proof of Proposition~\ref{prop:weighted-minimax-lcb}]

Picking $\delta = \frac{1}{\sqrt{|\bn|}}$ for the LCB algorithm, according to Corollary~\ref{coro:ub-lcb-instance}
gives that there exists a universal constant $c$~(which may contain the term $\log k$) such that 
$\cR(\lcbalg, \theta) \le \frac{c \sqrt{\log |\bn|} }{\sqrt{n_1}}$. 
Applying Proposition~\ref{prop:mean-est} concludes the proof.

\end{proof}

\begin{proof}[Proof of Proposition~\ref{prop:weighted-minimax-ucb-greedy}]

Consider a sequence of counts $\bn_1, \bn_2,...$ with $n_2 = 1$ and $n_1=2,3,..., +\infty$.
Fix $\mu_1 = \mu_2 + 0.1$
and let $\Delta=\mu_1 - \mu_2$. For the UCB algorithm, we have
\begin{align*}
\cR(\ucbalg, \theta) 
& = \Delta \Prbb{\theta}{ \hat{\mu}_2 + \frac{\beta_{\delta}}{\sqrt{n_2}} \ge \hat{\mu}_1 + \frac{\beta_{\delta}}{\sqrt{n_1}} } \\
& = 0.1 \Prbb{\theta}{ \hat{\mu}_1 - \hat{\mu}_2 \le  \frac{\beta_{\delta}}{\sqrt{n_2}} - \frac{\beta_{\delta}}{\sqrt{n_1}} }  \\
& \ge 0.1 \Prbb{\theta}{ \hat{\mu}_1 - \hat{\mu}_2 \le  \left(1 - \frac{1}{\sqrt{2}}\right) \beta_\delta } \\ 
& \ge 0.1 \Prbb{\theta}{ \hat{\mu}_1 - \hat{\mu}_2 \le  \left(1 - \frac{1}{\sqrt{2}}\right) } \\
& \ge 0.1 \Prbb{\theta}{ \hat{\mu}_1 - \hat{\mu}_2 \le  \Delta } \\
& = 0.05 
\end{align*}
where we applied the fact that $\beta_\delta \ge 1$ for any $\delta \in (0, 1)$ and the random variable $\hat{\mu}_1 - \hat{\mu}_2$ follows a Gaussian distribution with mean $\Delta$. Applying Proposition~\ref{prop:mean-est} gives 
\begin{align*}
\limsup_{j\rightarrow\infty} \sup_{\theta\in\Theta_{\bn_j}}{\frac{\cR(\textnormal{UCB},\theta)}{\sqrt{\log |\bn_j|}\cdot \cE^*(\theta)}} \ge \limsup_{j\rightarrow\infty} \frac{0.05 \sqrt{j + 1} }{c_1 \sqrt{\log (j + 2)}} = + \infty
\end{align*}
For the greedy algorithm, we have
\begin{align*}
\cR(\greedyalg, \theta)  = 0.1 \Prbb{\theta}{ \hat{\mu}_1 - \hat{\mu}_2 \le  0 } \,.
\end{align*}
The random variable $\hat{\mu}_1 - \hat{\mu}_2$ follows a Gaussian distribution with mean $\Delta > 0$ and 
variance $\frac{1}{n_1} + \frac{1}{n_2} \ge 1$.
Since shrinking the variance of  $\hat{\mu}_1 - \hat{\mu}_2$ will lower the probability $\Prbb{\theta}{ \hat{\mu}_1 - \hat{\mu}_2 \le  0 }$, we have $\cR(\greedyalg, \theta) \ge 0.1\Phi(-0.1)$ where $\Phi$ is the CDF for the standard normal distribution. Now using a similar statement as for the UCB algorithm gives the result.

\end{proof}

\end{appendix}

\end{document}